\documentclass{article}[11pt]

\usepackage{microtype}
\usepackage{graphicx}
\usepackage{subcaption} 
\usepackage{booktabs} % for professional tables
\usepackage{fullpage}
\usepackage[utf8]{inputenc}
\usepackage{mathtools, amsthm, amssymb}
\usepackage{xcolor}
\usepackage{url}
\usepackage{algorithm}
\usepackage{algorithmic}
\usepackage{natbib}
\usepackage{comment}
\usepackage{hyperref}
\usepackage{fancyhdr}

\def\calA{\mathcal{A}}
\def\calB{\mathcal{B}}
\def\calC{\mathcal{C}}

\def\calE{\mathcal{E}}
\def\calF{\mathcal{F}}
\def\calG{\mathcal{G}}

\def\calL{\mathcal{L}}

\def\calN{\mathcal{N}}

\def\calP{\mathcal{P}}

\def\calT{\mathcal{T}}

\def\calX{\mathcal{X}}
\def\calY{\mathcal{Y}}
\def\calZ{\mathcal{Z}}
\def\E{\mathbb{E}}
\def\R{\mathbb{R}}

\def\hsbm{\textnormal{HSBM}}
\def\dpcom{\textsf{DPCommunity}}
\def\dphchsbm{\textsf{DPClusterHSBM}}
\def\dphcblocks{\textsf{DPHCBlocks}}
\def\sparsecut{\textsf{SparseCut}}
\def\linkage{\textsf{Linkage}}

\def\dcost{\textnormal{cost}}

\def\cost{\omega}
\def\mw{\textsf{MW}}

\newtheorem{thm}{Theorem}
\newtheorem{lem}{Lemma}
\newtheorem{coro}{Corollary}
\newtheorem{defn}{Definition}

\newtheorem{fact}{Fact}

\title{Differentially-Private Hierarchical Clustering with Provable Approximation Guarantees}

\author{
Jacob Imola\footnote{Work partially done while at Google.}\\
UC San Diego \\
\href{mailto:jimola@eng.ucsd.edu}{jimola@eng.ucsd.edu}
\and 
Alessandro Epasto \\
Google \\
\href{mailto:aepasto@google.com}{aepasto@google.com}
\and
Mohammad Mahdian \\
Google \\
\href{mailto:mahdian@google.com}{mahdian@google.com}
\and
Vincent Cohen-Addad \\
Google \\
\href{mailto:cohenaddad@google.com}{cohenaddad@google.com}
\and
Vahab Mirrokni \\
Google \\
\href{mailto:mirrokni@google.com}{mirrokni@google.com}
}
\begin{document}

\maketitle
\begin{abstract}

%Hierarchical Clustering is a popular  unsupervised machine learning method with decades of history  and numerous applications. Despite its wide use, almost no work exists for {\it privacy-preserving} algorithms for the problem. 
 
%In this paper, we initiate the study of {\it differentially-private} approximation algorithms for hierarchical clustering under the rigorous optimization framework introduced by~\citet{dasgupta2016cost}. 

%On the theoretical side, first, we show that strong inapproximability lower-bounds exists for the problem under privacy constraints: any $\epsilon$-DP algorithm must exhibit $O(|V|^2/ \epsilon)$ additive error for input data set $V$. Then, we provide the first edge-level differentially-private algorithm for Dasgupta's cost with polynomial time and  matching the multiplicative approximation of the best non-private algorithm (but with additive error due to privacy).

%Then, motivated by the lower bounds for arbitrary graphs, we study a popular model of graphs with planted hierarchical clusters based on the stochastic block model. For graphs in such model, we present a $1+o(1)$ approximation algorithm recovering almost exactly the hierarchy on the ground-truth communities under separation assumptions on the clusters.

Hierarchical Clustering is a popular unsupervised machine learning method with decades of history and numerous applications.
We initiate the study of {\it differentially private} approximation algorithms for hierarchical clustering under the rigorous framework introduced by~\citet{dasgupta2016cost}. We show strong lower bounds for the problem: that any $\epsilon$-DP algorithm must exhibit $O(|V|^2/ \epsilon)$-additive error for an input dataset $V$. Then, we exhibit a polynomial-time approximation algorithm with $O(|V|^{2.5}/ \epsilon)$-additive error, and an exponential-time algorithm that meets the lower bound. To overcome the lower bound, we focus on the stochastic block model, a popular model of graphs, and, with a separation assumption on the blocks, propose a private $1+o(1)$ approximation algorithm which also recovers the blocks exactly. Finally, we perform an empirical study of our algorithms and validate their performance.
\end{abstract}

\section{Introduction}\label{sec:introduction}
Hierarchical Clustering is a staple of unsupervised machine learning with more than 60 years of history~\citep{ward1963hierarchical}.
%and numerous applications~\cite{leskovec2014mining,diez2015novel,tumminello2010correlation,dasgupta2016cost}. 
Contrary to {\it flat} clustering methods (such as $k$-means,~\citet{J10}), which provide a single partitioning of the data, {\it hierarchical} clustering algorithms produce a recursive refining of the partitions into increasingly fine-grained clusters. 
%More formally, a hierarchical clustering algorithm is given as input an arbitrary set $V$ of items, and a (arbitrary) notion of similarity (or dissimilarity) between them encoded as a weighted undirected graph $G=(V, E, w)$. Given this input, a hierarchical clustering algorithm outputs an unsupervised representation of the dataset as a tree (a.k.a. dendrogram) where the leaves  correspond to the input data points $V$ and the internal nodes of the tree correspond to the clusters at the different levels of granularities. 
%The objective of the tree, intuitively, is to cluster together the most similar items in the lowest possible clusters, while separating dissimilar items as high as possible in the tree.
The clustering process can be described by a tree (or dendrogram), and the objective of the tree is to cluster the most similar items in the lowest possible clusters, while separating dissimilar items as high as possible.

The versatility of such methods is apparent from the widespread use of hierarchical clustering in disparate areas of science, such as social networks analysis~\citep{leskovec2014mining,mann2008use}, bioinformatics~\citep{diez2015novel},  phylogenetics~\citep{sneath1962numerical,jardine1968model}, gene expression analysis~\citep{eisen1998cluster}, text classification~\citep{steinbach2000} and  finance~\citep{tumminello2010correlation}. Popular hierarchical clustering methods (such as linkage~\citep{J10}) are commonly available in standard scientific computing packages~\citep{virtanen2020scipy} as well as large-scale  production systems~\citep{bateni2017affinity, dhulipala2022hierarchical}.

Despite the fact that many of these applications involve private and sensitive user data, all research on hierarchical clustering (with few exceptions~\citep{kolluri2021private, xiao2014differentially} discussed later) has ignored the problem of defining {\it  privacy-preserving} algorithms. In particular, to the best of our knowledge, no work has provided {\it differentially-private (DP)}~\citep{dwork2014algorithmic} algorithms for hierarchical clustering with provable approximation guarantees. 

In this work, we seek to address this limitation by advancing the study of differentially-private approximation algorithms for hierarchical clustering under the rigorous optimization framework introduced by~\citet{dasgupta2016cost}. This celebrated framework introduces an objective function for hierarchical clustering (see Section~\ref{sec:preliminaries} for a formal definition) formalizing the goal of clustering similar items lower in the tree. 

%We provide approximation algorithms for the Dasgupta's cost function, while respecting a popular definition of differential privacy (DP) for graph data, known as edge-level DP~\cite{marc2021differentially,eliavs2020differentially, epasto2022differentially}. In this definition, two undirected graphs $G=(V,E,w')$ and $G'=(V,E',w')$, are {\it neighbors} if they differ only in the presence (or weight) of a single edge. An algorithm $\mathcal{A}$ is edge-level $\epsilon$-differentially private if the difference in the probability of observing any particular outcome from the algorithm when run on $G$ vs $G'$ is bounded: $\Pr[\mathcal{A}(G) \in S] \leq e^{\epsilon}\cdot \Pr[\mathcal{A}(G') \in S].$

%Algorithms respecting edge-level differential privacy, promise a strong notion of plausible deniability for the input data: an adversary observing the hierarchical clustering is information-theoretically bounded in the ability to determine the input similarity between any two items. This is especially relevant when the input graph represents private (or sensitive) user information (for example in social networks,~\citet{leskovec2014mining}). 

Our algorithms are edge-level {\em Differentially Private (DP)} on an input similarity graph, which is relevant when edges of the input graph represents sensitive user information.
Designing an edge-level DP algorithm requires proving that the algorithm is insensitive to changes to a single edge of the similarity graph. As we shall see, this is especially challenging for hierarchical clustering. In fact, commonly-used hierarchical clustering algorithms (such as linkage-based ones~\citep{J10}) are {\it deterministically} sensitive to a single edge, thus leaking directly the input edges. Moreover, as we show, strong inapproximability bounds exist for Dasgupta's objective under differential privacy, highlighting the technical difficulty of the problem.

\paragraph{Main contributions}
First, we show in Section~\ref{sec:lower-bounds} that no edge-level $\epsilon$-DP algorithm (even with exponential time) exists for Dasgupta's objective with less than $O(|V|^2/ \epsilon)$ additive error. This prevents defining private algorithms with meaningful approximation guarantees for {\it arbitrary} sparse graphs.

Second, on the positive side, we provide the first polynomial time, edge-level approximation algorithm for Dasguta's objective with $ O(|V|^{2.5} /\epsilon)$ additive error and multiplicative error matching that of the best non-private algorithm~\citep{agarwal2022sublinear}. This algorithm is based on recent advances in private cut sparsifiers~\citep{eliavs2020differentially}. Moreover, we show an (exponential time) algorithm with $O(|V|^{2} \log n/ \epsilon)$ additive error, almost matching the lower bound.

Third, given the strong lower bounds, in Section~\ref{sec:algorithms-hsbm} we focus on a popular model of graphs with a planted hierarchical clustering based on the {\em Stochastic Block Model (SBM)}~\citep{cohen2017hierarchical}. For such graphs, we present a private $1+o(1)$ approximation algorithm recovering almost exactly the hierarchy on the blocks. Our algorithm uses, as a black-box, any reconstruction algorithm for the stochastic block model. 

Fourth, we introduce a practical and efficient DP SBM community reconstruction algorithm (Section~\ref{sec:algorithms-hsbm}). This algorithm is based on perturbation theory of graph spectra combined with dimensionality reduction to avoid adding high noise in the Gaussian mechanism. Combined with our clustering algorithm, this results in the first private approximation algorithm for hierarchical clustering in the hierarchical SBM. 

Finally, we show in Section~\ref{sec:experiments} that this algorithm can be efficiently implemented and works well in practice.

\section{Related Work}
\label{sec:related}

Our work spans the areas of differential privacy, hierarchical clustering and community detection in stochastic block model. For a complete discussion, see Appendix~\ref{app:related}.

\paragraph{Graph algorithms under DP}
Differential privacy~\citep{dwork2006calibrating} has recently the gold standard of privacy. We refer to~\citet{dwork2014algorithmic} for a survey. Relevant to this work is the area of differential privacy in graphs. Definitions based on edge-level~\citep{epasto2022differentially,eliavs2020differentially} and node-level~\citep{kasiviswanathan2013analyzing} privacy have been proposed. The most related work is that on graph cut approximation~\citep{eliavs2020differentially,arora2019differentially}, as well as that of private correlation clustering~\citep{bun2021differentially, cohen2022near}.

\paragraph{Hierarchical Clustering}
Until recently, most work on hierarchical clustering were heuristic in nature, with the most well-known being the linkage-based ones~\citep{J10,bateni2017affinity}. \citet{dasgupta2016cost} introduced a combinatorial objective for hierarchical clustering which we study in this paper. Since this work, many authors have designed algorithms for variants of the problem with no privacy~\citep{cohen2017hierarchical,cohen2019hierarchical, charikar2017approximate,moseley2017approximation, agarwal2022sublinear,chatziafratis2020bisect}.

Limited work has been devoted to DP hierarchical clustering algorithms. One paper~\citep{xiao2014differentially} initiates private clustering via MCMC methods, which are not guaranteed to be polynomial time. Follow-up work~\citep{kolluri2021private} shows that sampling from the Boltzmann distribution (essentially the exponential mechanism~\citep{mcsherry2007mechanism} in DP) produces an approximation to the maximization version of Dasgupta's function, which is a different problem formulation. Again, this algorithm is not provably polynomial time.

\paragraph{Private flat clustering}
Contrary to hierarchical clustering, the area of private {\it flat} clustering on metric spaces has received large attention. Most work in this area has focused on improving the privacy-approximation trade-off~\citep{ghazi2020differentially,balcan2017differentially} and on efficiency~\citep{hegde2021sok,cohennear,cohen2022scalable}.

\paragraph{Stochastic block models}

The Stochastic Block Model (SBM) is a classic model for random graphs with planted partitions which has received a significant attention in the literature~\citep{MR3520025-Guedon16,montanari2016semidefinite, moitra2016robust,MR4115142,ding2022robust,Liu-Moitra-minimax}. For our work, we focus on a variant which has nested ground-truth communities arranged in hierarchical fashion. This model has received attention for hierarchical clustering~\citep{cohen2017hierarchical}.   

The study of private algorithms for SBMs is instead very recent. One of the only results known for private (non-hierarchical) SBMs is the work of~\citet{seif2022differentially} which provides quasi-polynomial time community detection algorithms for some regimes of the model.
Finally, concurrently to our work, the manuscript of~\citet{chen2023private} provides strong approximation guarantees using semi-definite programming for recovering SBM communities. Community detection is a distinct problem from hierarchical clustering; however as explained below, we use community detection as a sub-routine for our clustering algorithm.

No results are known for approximating hierarchical clustering on hierarchical SBMs. For this reason, in Section~\ref{sec:algorithms-hsbm} we design a hierarchical clustering algorithm (Algorithm~\ref{alg:priv-hc-hsbm}) which uses community detection as a black-box. Moreover, we show a novel algorithm for  hierarchical SBM community detection (Algorithm~\ref{alg:priv-hsbm}), independent of~\citet{chen2023private}, which is of practical interest because it uses SVDs, instead of semidefinite programming, and thus does not have a large polynomial run-time.

\section{Preliminaries}\label{sec:preliminaries}
Our results involve the key concepts of hierarchical clustering and differential privacy. We define these two concepts in the next sections.

\subsection{Hierarchical Clustering}

Hierarchical clustering seeks to produce a tree clustering a set $V$ of $n$ items by their similarity. It takes as input an undirected graph $G = (V, E, w)$, where $E \subseteq V \times V$ is the set of edges and $w : V \times V \rightarrow \R^+$ is a weight function indicating similarity; i.e. a higher $w(u,v)$ indicates $u,v$ are more similar. We extend the weight function $w$ and say that $w(u,v) = 0$ if $w(u,v) \notin E$.

A hierarchical clustering (HC) of $G$ is a tree $T$ whose leaves are $V$. The tree can be viewed as a sequence of merges of subtrees of $T$, with the final merge being the root node. A good hierarchical clustering merges more similar items closer to the bottom of the tree. The cost function $\cost_G(T)$ of Dasgupta~\citep{dasgupta2016cost}, captures this intuition. We have
\begin{equation}\label{eq:hc-cost}
    \cost_G(T) = \sum_{(u,v) \in V^2} w(u,v) |\text{leaves}(T[u\wedge v])|,
\end{equation}
where $T[u \wedge v]$ indicates the smallest subtree containing $u,v$ in $T$ and $|\text{leaves}(T[u \wedge v])|$ indicates the number of leaves in this subtree. This cost function charges a tree $T$ for each edge based on the similarity $w(u,v)$ and how many leaves are in the subtree in which it is merged. 

\paragraph{Additional Notation}
We let $\cost_G^* = \min_T \omega_G(T)$ denote the best possible cost attained by any tree $T$. We write $w(A,B) = \sum_{a \in A, b \in B} w(a,b)$ and we say $w(G) = w(G,G)$. Let $\calA(G)$ be a hierarchical clustering algorithm. We say $\calA$ is an $(a_n, b_n)$-approximation if
\begin{equation}\label{eq:hc-obj}
    \E[\cost_G(\calA(G))] \leq a_n \cost_G^* + b_n,
\end{equation}
where the expectation is over the random coins of $\calA$. %We define a similar notion for the revenue objective, replacing~\eqref{eq:hc-obj} with $\E[\cost_G^{\mw}(\calA(G))] \geq a_n \cost_G^{\mw *} - b_n$.

\subsection{Differential Privacy}
For hierarchical clustering we use the notion of graph privacy known as edge differential privacy. Intuitively, our private algorithm behaves similarly whether or not the adjacency matrix of $G$ is altered in $L_1$ distance by up to $1$. Specifically, we say $G = (V, E, w)$ and $G' = (V, E', w')$ are \emph{adjacent graphs} if $\sum_{u,v \in V} |w(u,v) - w'(u,v)| \leq 1$, meaning that the adjacency matrices have $L_1$ distance at most one \footnote{the constant one may be changed to any constant to match the application, and our results carry over easily.}. This notion has been used before by~\citet{ eliavs2020differentially, blocki2012johnson} and it has many real-world applications, such as when the graph is a social network and the edges between users encode relationships between them~\citep{epasto2022differentially}.
The definition of edge-DP is as follows:
\begin{defn}\label{def:privacy}
An algorithm $\calA : \calG \rightarrow \calY$ satisfies $(\epsilon, \delta)$-edge DP if, for any $G = (V,E,w), G' = (V,E',w')$ that are adjacent, and any set of trees $\calT$,
\[
    \Pr[\calA(G') \in \calT ] \leq e^\epsilon \Pr[\calA(G) \in \calT] + \delta.
\]
\end{defn}
Edge DP states that given any output $\calT$ of $\calA$, it is provably hard to tell whether an adjacent $G$ or $G'$ was used.
For 0/1 weighted graphs, Definition~\ref{def:privacy} is equivalent to standard edge DP for unweighted graphs (c.f. Definition 2.2.1 in \cite{pinot2018minimum}).

\section{Lower Bounds}\label{sec:lower-bounds}
We show that for the both objective functions considered, there are unavoidable lower bounds on the objective function for any differentially private algorithm.
Our theorem applies a packing-style argument~\citep{hardt2010geometry}, in which we construct a large family $\calF$ of graphs such that no tree can cluster more than one graph in $\calF$ well. However, a DP algorithm $\calA$ is forced to place mass on all trees. This limits its utility as significant mass must be placed on trees which do not cluster the input graphs well. Formally, we prove the following theorem:

\begin{thm}\label{thm:lower-bound}
For any $\epsilon \leq \frac{1}{20}$ and $n$ sufficiently large, let $\calA(G)$ be a hierarchical clustering algorithm which satisfies $\epsilon$-edge differential privacy. Then, there is a weighted graph $G$ with $\cost_G^* \leq O(\frac{n}{\epsilon})$ such that 
\[
    \E[\cost_{G}(\calA(G))] \geq \Omega(\tfrac{n^2}{\epsilon}).
\]
\end{thm}

We prove this theorem in Section~\ref{sec:lb-proof}; we discuss the implications of the theorem here. Since there exists a graph such that $\cost_G^* \leq O(\frac{n}{\epsilon})$, yet $\cost_G(\calA(G)) \geq \Omega(\frac{n^2}{\epsilon})$, this means that no differentially private algorithm $\calA$ can be a $(O(n^{\alpha}), O(\frac{n^{2\alpha}}{\epsilon}))$ approximation to hierarchical clustering for any $\alpha < 1$. It is possible for $\calA$ to be a $(1, O(\frac{n^2}{\epsilon}))$-approximation--- in this case, for graphs with $W$ total weight, it easy to see that $\cost_G^* \leq O(nW)$ and can be as small as $O(W)$. Thus, it is necessary for $W$ to be much bigger than $\frac{n}{\epsilon}$, meaning that $G$ cannot be too sparse.

\subsection{Proof of Theorem~\ref{thm:lower-bound}}\label{sec:lb-proof}

To construct our lower bound, we consider the family of graphs $\calP(n, 5)$ consisting of $\frac{n}{5}$ cycles of size $5$. We observe the following facts:
\begin{itemize}
    \item Each $G \in \calP(n, 5)$ has $n$ edges. Thus, any $G_1, G_2 \in \calP(n, 5)$ differ in at most $2n$ edges.
    \item For any $G \in \calP(n, 5)$, any binary tree which splits the graph into its cycles before splitting any edges in the cycles incurs a cost of at most $\frac{n}{5} W_5$, where $W_5 = \cost_{C_5}^* \leq 18$.
\end{itemize}

It will be convenient to use the following definition:
\begin{defn}
For a graph $G$, a \emph{balanced cut} is partition $(A,B)$ of $V$ such that $\frac{n}{3} \leq |A|, |B| \leq \frac{2n}{3}$.
\end{defn}
Any hierarchical clustering $T$ can be mapped to a balanced cut on $G$ in the following way:
\begin{defn}
For a binary tree $T$ whose leaves are $V$, let the sequence $N_0, N_1, \ldots, N_r$ denote a recursive sequence of internal nodes such that $N_0$ is the root node, and $N_i$ is child of $N_{i-1}$ with more leaves in its subtree. Finally, $N_r$ is the first node in the sequence with fewer than $\frac{2n}{3}$ leaves in its subtree. Then, the balanced cut $(A,B)$ of $T$ is the partition $(\text{leaves}(N_r), V \setminus{leaves}(N_r))$.
\end{defn}
It is easy to see that $(A,B)$ is indeed a balanced cut of $G$, and for any edge $(u,v)$ crossing $(A,B)$, we have $|\text{leaves}(T[u \wedge v])| \geq \frac{2n}{3}$.

Our class $\calC$ of graphs is a subset of $\calP(n,5)$ for which no tree clusters more than one element of $\calC$ well. We characterize a condition for which a tree $T$ definitely does not cluster $G \in \calP(n,5)$ well:

\begin{defn}
    For a binary tree $T$, let $(A,B)$ be its balanced cut. We say $(A,B)$ \emph{misses} a cycle $C \subseteq G$ if at least one vertex of $C$ lies in $A$ and at least one vertex lies in $B$.
\end{defn}
Now, we show that if $T$ misses many cycles in its balanced cut, it must incur high cost.

\begin{lem}\label{lem:clique-miss-bad}
    For a graph $G \in \calP(n,5)$, let $T$ be a HC with balanced cut $(A,B)$, and suppose that $B$ misses at least $\alpha \frac{n}{5}$ of the cycles in $G$, for $0 < \alpha \leq 1$. Then,
    \begin{align*}
        \cost_G(T) &\geq \frac{4\alpha}{15} n^2.
    \end{align*}
\end{lem}
\noindent \textit{Proof:}
From the given information, we have that $w(A,B) \geq 2 \alpha \frac{n}{5}$, as a missed cycle implies at least two edges are cut. Thus,
\begin{align*}
    \cost_G(T) &\geq \sum_{u \in A, v \in B} w(u,v) |\text{leaves}(T[u \wedge v])| \\
    &\geq \tfrac{2n}{3} w(A,B) \geq \tfrac{4\alpha }{15} n^2. \qed
\end{align*} 

We generate graphs from $\calP(n, 5)$ at random, showing that the probability that there exists a balanced cut $(A,B)$ which misses few cycles in both $G_1, G_2$ is exponentially small. This will allow us to generate a large family of graphs such that no balanced cut misses few cycles in more than one graph. This results in the following lemma---in the following, let $\calB(G, r) = \{T \in \calT_n : \cost_G(T) < r\}$.
\begin{lem}\label{lem:packing-2}
    For $n$ sufficiently large, there exists a family $\calF \subseteq \calP(n,5)$ of size $2^{0.2n}$ such that $\calB(G, r) \cap \calB(G', r) = \emptyset$ for any $G,G' \in \calF$ with $r = \frac{n^2}{400}$.
\end{lem}

The proof of this lemma appears in Appendix~\ref{app:lower-bounds}.
Thus, no tree can cluster more than one of our random graphs well, and we can apply the packing argument to obtain Theorem~\ref{thm:lower-bound}.
We prove it as follows.

\noindent \textit{Proof of Theorem~\ref{thm:lower-bound}:} 
Let $\calF$ be the set of graphs guaranteed by Lemma~\ref{lem:packing-2}. We have $|\calF| = 2^{0.2n}$. Let $\calF_W$ contain the same graphs of $\calF$, but with each edge weighted by a positive integer $W$ satisfying $0.02 \leq \epsilon W < 0.07$. Each $G,G' \in \calF$ differs by up to $2n$ edges, and applying group privacy $W$ times, we have that an algorithm $A$ which satisfies $\epsilon$-DP satisfies $2n W \epsilon$-DP on the graphs in $\calF_W$.

Now, suppose $A$ satisfies $\E[\dcost_G(A(G))] < \frac{W}{800} n^2$ for any $G \in \calF_W$. This implies $\Pr[\dcost_G(A(G)) \in \calB(G, \frac{W}{400} n^2)] \geq \frac{1}{2}$ for all $G \in \calF_W$. However, we know these balls are disjoint because of the disjointness property on $\calF$. Furthermore, we have that $\Pr[A(G) \in \calB(G', \frac{W}{400} n^2)] \geq e^{-2nW\epsilon} \frac{1}{2} > 2^{-0.2n}$ for all $G' \in \calF_W$.
\begin{align*}
    1 &\geq \sum_{G' \in \calF_W} \Pr[A(G) \in \calB(G', \tfrac{W}{400} n^2)] \\
    &> 2^{0.2 n} 2^{-0.2n} = 1.
\end{align*}
This is a contradiction, and thus the algorithm $A$ must have error higher than $\frac{W}{800}n^2 \geq \Omega(\frac{n^2}{\epsilon})$ on some graph. \qed

\section{Algorithms for Private Hierarchical Clustering}\label{sec:algorithms}
In this section, we design private algorithms for hierarchical clustering which work on any input graph. In Section~\ref{sec:hc-poly}, we propose a polynomial time $(\alpha, O(\frac{n^{2.5}}{\epsilon}))$ approximation algorithm, where $\alpha$ is the best approximation ratio of a black-box, \emph{non-private} hierarchical clustering algorithm. Then, in Section~\ref{sec:hc-exp}, we show that the exponential mechanism is a $(1, O(\frac{n^{2} \log n}{\epsilon}))$-approximation algorithm, implying our lower bound is tight. The proofs of the results in this section appear in Appendix~\ref{app:hc-cuts}

\subsection{Polynomial-Time Algorithm}~\label{sec:hc-poly}

Our algorithm makes use of a recent algorithm which releases a sanitized, synthetic graph $G'$ that approximates the cuts in the private graph $G$~\citep{eliavs2020differentially,arora2019differentially}. Via post-processing, it is then possible to run a non-private, black-box clustering algorithm. We are able to relate the cost in $G'$ to that of $G$ by reducing the cost $\cost_G(T)$ to a sum of cuts. We start by defining the notion of $G'$ approximating the cuts in $G$.
\begin{defn}\label{def:cut-approx}
    For a given graph $G = (V, E, w)$, we say $G' = (V, E', w')$ is an $(\alpha_n, \beta_n)$-approximation to cut queries in $G$ if for all $S \subseteq V$, we have
    \[
        (1-\alpha_n) w(S, \overline{S}) - \beta_n \min \{|S|, n - |S|\} \\ \leq w'(S, \overline{S}) \leq (1+\alpha_n) w(S, \overline{S}) + \beta_n \min \{|S|, n - |S|\}.
    \]
\end{defn}

As we alluded, earlier work shows that it is possible to release an $(\tilde{O}(\frac{1}{\epsilon\sqrt{n}}), \tilde{O}(\frac{\sqrt{n}}{\epsilon}))$-approximation to cut queries while satisfying differential privacy. Using this result, we are able to run any blackbox hierarchical clustering algorithm, and by post-processing, the final clustering $T'$ will still satisfy privacy. Even though $T'$ is computed only viewing $G'$, we are able to relate $\cost_{G}(T')$ to $\cost_G^*$ using the fact that $G'$ approximates the cuts in $G$, and a decomposition of $\cost_{G'}(T')$ into a sum of cuts. This idea recently appeared in~\citet{agarwal2022sublinear}, and is a critical component of our theorem. In the end, we obtain the following:
\begin{thm}
Given an $(a_n, 0)$-approximation to the cost objective of hierarchical clustering, 
there exists an $(\epsilon, \delta)$-DP algorithm which, with probability at least $0.8$, is a  $((1+o(1))a_n, O(n^{2.5} \frac{ \log^2 n \log^2 \frac{1}{\delta}}{\epsilon}))$-approximation algorithm to the cost objective.
\end{thm}

Plugging in a state-of-the-art, $\sqrt{\log n}$ hierarchical clustering algorithm of~\citet{charikar2017approximate}, we obtain a $((1+o(1)) \sqrt{\log n}, \tilde{O}(\frac{n^{2.5}}{\epsilon}))$-approximation. In a graph with total edge weight $W$, we have $W \leq \cost_G(T) \leq nW$, and thus an approximation is possible if $W > \frac{n^{1.5}}{\epsilon}$. This means the graph can have an average degree of $\frac{\sqrt{n}}{\epsilon}$.

\subsection{Exponential Mechanism}~\label{sec:hc-exp}
We consider an algorithm based on the well-known exponential mechanism~\citep{mcsherry2007mechanism}. This algorithm takes exponential time, but achieves greater performance that is nearly tight with our lower bound (showing that the lower bound can't be improved significantly from an information-theoretic point of view).

The exponential mechanism $M : \calX \rightarrow \calY$ releases an element from $\calY$ with probability proportional to 
\[
    \Pr[M(X) = Y] \propto e^{\epsilon u_X(Y) / (2S)},
\]
where $u_X(Y)$ is a utility function, and $S = \max_{X,X',Y}|u_X(Y) - u_{X'}(Y)|$ is the sensitivity of the utility function in $X$. This ubiquitous mechanism satisfies $(\epsilon, 0)$-DP. 

In our setting, we use the utility function $u_G(T) = -\cost_G(T)$.
The sensitivity is bounded in the following fact.
\begin{fact}
For two adjacent input graphs $G = (V,E,w)$ and $G' = (V,E,w')$, we have for all trees $T$ that $|\cost_{G}(T) - \cost_{G'}(T)| \leq n$.
\end{fact}
\noindent\textit{Proof:}
    We can write the difference as as 
    \begin{align*}
        &|\cost_G(T) - \cost_{G'}(T)| \\ 
        &= \textstyle{\left|\sum_{u,v \in V^2} (w(u,v) - w'(u,v))|\texttt{leaves}(T[u\wedge v])| \right|} \\
        &\leq \textstyle{\sum_{u,v \in V^2} |w(u,v) - w'(u,v))| \cdot |\texttt{leaves}(T[u\wedge v])|} \\
        &\leq n \textstyle{\sum_{u,v \in V^2} |w(u,v) - w'(u,v)| \leq n}. \hfill \qed
    \end{align*}
Having controlled the sensitivity, we can apply utility results for the exponential mechanism.
\begin{lem}\label{lem:exp-util}
There exists an $(\epsilon, 0)$-DP, $(1, O(\frac{n^2 \log n}{\epsilon}))$-approximation algorithm for hierarchical clustering.
\end{lem}

Thus, the exponential mechanism improves on the cost, and shows that private hierarchical clustering can be done on graphs with average degree $O(\frac{n}{\epsilon})$.

\section{Private Hierarchical Clustering in the Stochastic Block Model}\label{sec:algorithms-hsbm}

In this section, we propose a hierarchical clustering algorithm designed for input graph generated from the hierarchical stochastic block model (HSBM), a graph model with planted communities arranged in a hierarchical structure. We define this model in Section~\ref{sec:hsbm-setup}. Next, in Section~\ref{sec:priv-hc-hsbm}, we outline \dphcblocks{}, a lightweight private hierarchical clustering algorithm in the HSBM, which uses community detection as a black box. This approach enables any DP community detection algorithm to be used as a sub-routine. Finally, in Section~\ref{sec:dp-comm-detection}, we propose a practical, private community detection algorithm which is the first to work in the general HSBM. Combining the results in Sections~\ref{sec:priv-hc-hsbm} and~\ref{sec:dp-comm-detection}, we obtain a private, $1+o(1)$-approximation algorithm to the Dasgupta cost function.

\subsection{Hierarchical Stochastic Block Model of Graphs}\label{sec:hsbm-setup}

In this section, we consider unweighted graphs $(V,E)$ where each edge has weight $1$. Observe that differential privacy (Definition~\ref{def:privacy}) corresponds to adding or removing an edge from $G$. In the HSBM~\citep{cohen2017hierarchical}, there is a partition of $V$ into blocks (communities) $B_1, B_2, \ldots, B_k$ of $V$ with the properties that two items in the same block have the same set of edge probabilities, and that items in different blocks are less likely to be connected with these probabilities following a hierarchical structure.

The probabilities of the edges in $B$ are specified by a tree $P$ with leaves $B = B_1, \ldots, B_k$, internal nodes $N$, and a function $f : N \cup B \rightarrow [0,1]$. To capture the decreasing probability of edges,  $f$ must satisfy $f(n_1) < f(n_2)$ whenever $n_1$ is an ancestor of $n_2$ in $P$. Formally, we have~\citep{cohen2017hierarchical}

\begin{defn}
Let $B = B_1, \ldots, B_k$; $P$ be a tree with leaves in $B$ and internal nodes $N$; and $f: N \cup B \rightarrow [0,1]$ be a function satisfying that $f(n_1) < f(n_2)$ whenever $n_1$ is an ancestor of $n_2$ in $P$. We refer to the triplet $(B,P,f)$ as a ground-truth tree. Then, $\hsbm(B,P,f)$ is a distribution over graphs $G$ whose edges are drawn independently, such that for $u,v \in P$, we have 
\[
\Pr[(u,v) \in G] = f(LCA_P(B_u, B_v)),
\]
where $LCA_P$ denotes the least common ancestor of the blocks $B_u, B_v$ containing $u,v$ in $P$. 
\end{defn}

Due to the randomness of the graph $G$, it would be unreasonable to expect to be able to recover the exact $(B,P,f)$ from $G$. Our algorithms will recover an approximate ground-truth tree, according to the following definition:

\begin{defn} (From~\citet{cohen2017hierarchical}):
Let $(B,P,f)$ be a ground-truth tree, and let $(B,T,f')$ be another ground-truth tree with the same set of blocks. We say $(B,T, f')$ is a $\gamma$ approximate ground-truth tree if for all $u, v \in B$, $
\gamma^{-1} f(LCA_P(u,v)) \leq f'(LCA_{P'}(u,v)) \leq \gamma f(LCA_P(u,v))$.
\end{defn}
For $\gamma \approx 1$, an approximate ground-truth tree means that $\hsbm(B,P,f)$ and $\hsbm(B,P',f')$ are essentially the same distribution.

\subsection{Producing a DP HC given the communities}\label{sec:priv-hc-hsbm}

Given the blocks (communities) of an HSBM, we now propose \dphcblocks{}, a lightweight, private algorithm for returning a $1+o(1)$-approximation to the Dasgupta cost. Our algorithm uses some ideas from the non-private algorithm proposed in~\citet{cohen2017hierarchical,cohen2019hierarchical}.

\dphcblocks{} takes in $G$ generated from $\hsbm(B,P,f)$, as well as the blocks $B$. 
To produce an approximate ground-truth tree, it considers similarities $sim(B_i, B_j) = \frac{w_G(B_i, B_j)}{|B_i||B_j|}$ for every pair of blocks. It then performs a process similar to single linkage: until all blocks are merged, it greedily merges the groups with the highest similarity, and considers the similarity between this new group and any other groups to be the maximum similarity of any pair of blocks between the groups. Privacy comes from addition of Laplace noise in the similarity calculation, which is the only place in which the private graph $G$ is used. \dphcblocks{} appears as Algorithm~\ref{alg:priv-hc-hsbm}.

\dphcblocks{} accesses the graph via the initial similarities $sim(B_i, B_j)$. By observing the sensitivity $\max_{B_i, B_j} |w_{G'}(B_i, B_j) - w_G(B_i, B_j)|$ is at most $1$, we are able to prove its privacy. We also use the fact that adding an edge can only affect $sim(B_i, B_j)$ for just one choice of $B_i, B_j$.
\begin{thm}
\dphcblocks{} satisfies $\epsilon$-edge DP in the parameter $G$.
\end{thm}
\begin{proof}
Observe the algorithm can be viewed as a post-processing of the set $\calB = \{sim(B_i, B_j) + \calL_{ij} : i,j \in k\}$ where $\calL_{ij} \sim Lap(\frac{1}{\epsilon})$ i.i.d. Suppose an edge is added between $B_{i}, B_{j}$. Then, $sim(B_{i}, B_{j}) + \calL_{ij}$ is protected by $\epsilon$-edge DP by the Laplace mechanism, observing the sensitivity of $w_G(B_{i}, B_{j})$ is $1$. The other quantities in $\calB$ follow the same distribution, so $\calB$ itself satisfies $\epsilon$-edge DP.
\end{proof}
We stress that, crucially,
Algorithm~\ref{alg:priv-hc-hsbm} and all our algorithms are DP for any input graph $G$, even if the graphs do not come from the HSBM model. We will use the input distribution assumptions only in the utility proofs.

We are also able to show a utility guarantee that \dphcblocks{} is a $(1+o(1), 0)$-approximation to the cost objective. In order to prove this, we need to assume that the blocks in the HSBM are sufficiently large (at least $n^{2/3}$) and that the edge probabilities are at least $\frac{\log n}{\sqrt{n}}$. These assumptions are necessary to ensure concentration of the graph cuts between blocks, so that an accurate approximate tree may be formed. Also, it requires that $\epsilon \geq \frac{1}{\sqrt{n}}$---this is an extremely light assumption, and it still permits us to use a small, constant value of $\epsilon$ to guarantee strong privacy. Formally,

\begin{thm}\label{thm:hc-hsbm-util}
For $\epsilon \geq \frac{1}{\sqrt{n}}$ and a graph $G$ drawn from $\hsbm(B, P, f)$ such that $|B_i| \geq n^{2/3}$ and $f \geq \frac{\log n}{\sqrt{n}}$, with probability $1-\frac{2}{n}$, the tree $T$ outputted by \dphcblocks{} satisfies
$\cost_G(T) \leq (1 + o(1)) \cost_G(T')$.
\end{thm}
In fact, we show a stronger result that the tuple $(B,T,f')$ returned by \dphcblocks{} is a $1+o(1)$-approximate ground-truth tree for $\hsbm(B,P,f)$. By a result from~\citet{cohen2019hierarchical}, this implies it achieves the approximation guarantee. We defer the proof to Appendix~\ref{sec:hc-hsbm-util}.

\begin{algorithm}
\caption{\dphcblocks{}, a hierarchical clustering algorithm in the HSBM given the blocks.}\label{alg:priv-hc-hsbm}
\begin{algorithmic}
\STATE{\textbf{Input:} $G = (V,E)$ drawn from the HSBM; blocks $B_1, \ldots B_k$ partitioning $V$, privacy parameter $\epsilon$}
\STATE{\textbf{Output:} Tree $T$.}
\FOR{$i = 1$ to $k$}
    \STATE{$T_i$ is a random HC with leaves $B_i$}
\ENDFOR
\STATE{$sim(B_i, B_j) \gets \frac{w_G(B_i, B_j) + \calL_{ij}}{|B_i||B_j|}$, where $\calL_{ij} \sim Lap(\frac{1}{\epsilon})$.}
\STATE{$\calC = \{B_1, \ldots, B_k\}$}
\STATE{$T = forest(T_1, \ldots, T_k)$}
\WHILE{$|\calC| \geq 1$}
\STATE{$A_1, A_2 = \arg \max_{A_1, A_2 \in \calC} sim(A_1, A_2)$}
\STATE{Merge $A_1,A_2$ in $T$; $C = A_1 \cup A_2$}
\STATE{$f'(C) = sim(A_1, A_2)$}
\STATE{$\calC = (\calC \setminus \{A_1, A_2\}) \cup \{C\}$}
\FOR{$S \in \calC \setminus \{C\}$:}
\STATE{$sim(S,C) \gets \max_{B_i \in S, B_j \in C} sim(B_i, B_j)$}
\ENDFOR
\ENDWHILE
\STATE{\textbf{Return:} $(B,T,f')$.}
\end{algorithmic}
\end{algorithm}

\subsection{DP Community Detection in the HSBM}\label{sec:dp-comm-detection}

We now develop a DP method of identifying the blocks $B$ of graph drawn from the HSBM. Combined with our clustering algorithm $\dphcblocks{}$, this forms an end-to-end algorithm for hierarchical clustering in the HSBM in which the communities are not known. 

In order to describe our algorithm, \dpcom{}, we introduce some notation. For a model $\hsbm{}(B,P,f)$, we associate an $n \times n$ expectation matrix $A$ given by the probabilities that edge $(i,j)$ appears in $G$. We then let $\hat{A}$ be a randomized rounding of $A$ to $\{0,1\}$ which is simply the adjacency matrix of $G$. \dpcom{} recovers communities when they are separated in the sense defined by 
\[
    \Delta = \min_{u \in B_i, v \in B_j : i \neq j} \|A_u - A_v\|_2,
\]
where $A_u$ is the $u$th column of $A$. Next, we let $\sigma_1(A), \ldots, \sigma_n(A)$ denote the singular values of $A$ in order of decreasing magnitude. Finally, we let $\Pi_A^{(k)}$ denote the projection onto the top $k$ left singular values of $A$---formally, if $U_k$ consists of the top $k$ singular values of $A$, then $\Pi_A^{(k)} = U_kU_k^T$.

\dpcom{} is given the adjacency matrix $\hat{A}$ of a graph drawn from $\hsbm(B,P,f)$, as well as $k$, the number of blocks. In practice, $k$ may be treated as a hyperparameter to be optimized. \dpcom{} uses the spectral method~\citep{mcsherry2001spectral, vu2014simple} to cluster the columns of $\hat{A}$. These results show that the columns in $F = \Pi_{\hat{A}}^{(k)}(\hat{A})$ forms a clustering of the points into their original blocks. To make this private, we use stability results of the SVD to compute (an upper bound of) the sensitivity $\Gamma$ of $F$, and add noise $N$ via the Gaussian mechanism. Since $N,F$ are both $n \times n$ matrices, the $l_2$ error introduced by $N$ grows with $\sqrt{n}$, which is large. Our final observation is that, since the distances in $F$ are all that matter, we may project $F$ to $\log(n)$-dimensional space using Johnson-Lindenstrauss~\citep{johnson1984extensions}, and then add Gaussian noise whose error grows with $\sqrt{\log n}$. \dpcom{} is shown in Algorithm~\ref{alg:priv-hsbm}.

There are two important remarks about \dpcom{}. First, to ensure an accurate, private upper bound on $\Gamma$, we need the mild assumption that the spectral gap $\sigma_k(\hat{A}) - \sigma_{k+1}(\hat{A})$ is not too small, and if it is, the algorithm returns $\perp$. For most choices of parameters in the SBM, the spectral gap is always much larger than needed---the check is only to ensure privacy even for input graphs not from the SBM. Second, due to ease of theoretical analysis, $\hat{A}$ is split into two parts, and one part is projected onto the top $k$ singular values of the other. This removes probabilistic dependence between variables, but the high level ideas are the same.

\begin{algorithm}
\caption{\dpcom{}, a community recovery Algorithm}\label{alg:priv-hsbm}
\begin{algorithmic}
\STATE{\textbf{Input:} $\hat{A}$, adjacency matrix generated from $\hsbm(B,P,f)$, privacy parameter $\epsilon$.}
\STATE{\textbf{Output:} $f_z$, an estimate of blocks on a set $Z_2 \subseteq V$.}
\STATE{Compute a random partition $Y \sqcup Z_1 \sqcup Z_2$ of $V$ such that $|Y| = \frac{n}{2}$, $|Z_1| = |Z_2| = \frac{n}{4}$.}
\STATE{$\tilde{A}_1 \gets \tilde{A}_{YZ_1}$ (submatrix of $\hat{A}$ with rows $Y$, cols. $Z_1$).}
\STATE{$\tilde{A}_2 \gets \tilde{A}_{YZ_2}$}
\STATE{$\tilde{d}_k \gets \sigma_{k}(\hat{A}_1) - \sigma_{k+1}(\hat{A}_1) - \frac{8}{\epsilon}\ln \frac{4}{\delta} + Lap(\frac 8 \epsilon)$}
\STATE{$\tilde{\sigma}_1 \gets \sigma_1(\hat{A}_2) + \frac{4}\epsilon \ln \frac{4}{\delta} + Lap(\frac{4}{\epsilon})$}
\IF{$\hat{d}_k \leq 10(\tfrac{8}{\epsilon} \ln \tfrac{4}{\delta})$}
\RETURN{$\perp$}
\ENDIF
\STATE{$\tilde{\Gamma} \gets \frac{\tilde{\sigma}_1}{\hat{d}_k}, m \gets 64 \ln \frac{2n}{\delta}$.}
\STATE{$F \gets P \Pi_{\hat{A}_1}^{(k)}(\hat{A}_2)$, where $P \sim \calN(0, \frac 1 {\sqrt{m}})^{m \times n/2}$.}
\STATE{$\tilde{F} \gets F + N$, where $N \sim \frac{3k\tilde{\Gamma}}{\epsilon} \sqrt{2\ln \tfrac{5}{\delta}} \calN(0,1)^{m \times n/4}$.}
\RETURN{$\hat{F}$}
\end{algorithmic}
\end{algorithm}
We now analyze privacy and utility. Full proofs of the results in this section appear in Appendix~\ref{sec:algorithms-hsbm}.
Our privacy analysis involves analyzing the release of the singular values $\sigma_1, \sigma_k, \sigma_{k+1}$, and $\tilde{F}$. The bulk of this analysis comes from analyzing the sensitivity of $\tilde{F}$, which uses the accuracy of the Johnson-Lindenstrauss transform and spectral perturbation bounds.
\begin{thm}\label{thm:com-hsbm-priv}
(Privacy): For $\epsilon < 1$, Algorithm~\ref{alg:priv-hsbm} satisfies $(\epsilon, \delta)$-DP with respect to a change of one edge in $\hat{A}$.
\end{thm}
To prove the utility of \dpcom{}, we prove that recovery is possible provided that $\Delta$ is larger than some threshold depending on $\epsilon$, the singular values of $A$, the minimum edge probability, and the minimum block size, along with other mild assumptions on $k$ and the block sizes. These assumptions are necessary, as there will be too little data for concentration otherwise. Formally,
\begin{thm}\label{thm:com-hsbm-util-inf}
(Utility): 
Let $\hat{A}$ be drawn from $\hsbm(B,P,f)$, $\tau = \max f(x)$, and $s = \min_{i=1}^k |B_i|$. There is a universal constant $C$ such that if $\tau \geq C \frac{\log n}{n}$, $s \geq C \sqrt{n \log n}$, $k < n^{1/4}$, $\delta < \frac{1}{n}$, $\sigma_k(A) \geq C \max\{ \sqrt{n\tau}, \frac{1}{\epsilon} \ln \frac{4}{\delta}\}$, and 
\[\Delta > C\max\left\{\tfrac{ k (\ln \frac{1}{\delta})^{3/2}}\epsilon \tfrac{\sigma_1(A)}{\sigma_{k}(A)},  \sqrt{\tfrac{n\tau}{s}} + \sqrt{k\tau \log n} + \tfrac{\sqrt{nk\tau}}{\sigma_k}\right\},\]
then with probability at least $1 - 3n^{-1}$, \dpcom{} returns a set of points $\tilde{F} = \{f_i : i \in Z_2\}$ such that
\begin{align*}
    \|f_{i} - f_{j}\|_2 &\leq \tfrac{2\Delta}{5}  \ \ \ \text{if $\exists u.~i,j \in B_u$} \\
    \|f_{i} - f_{j}\|_2 &\geq \tfrac{4\Delta}{5}  \ \ \ \text{otherwise}.
\end{align*}
\end{thm}
Thus, if the assumptions are met, then $\tilde{F}$ consists of $k$ well-separated clusters which indicate the communities of each point in the sampled set $Z_2 \subset V$. These communities can be found using a simple routine such as $k$-centers. In order to cluster all of $V$, we can simply divide the privacy budget into $\log n$ parts, run \dpcom{} $\log n$ times, and merge the clusters.

To illustrate our theorem in a simple example, consider the HSBM with $k$ equal-sized blocks, and let $f_P(n) = p$ when $n$ is a parent of a leaf in $P$, and $f_P(n) = q$ otherwise, with $p \geq q$. This corresponds to probability $p$ of an edge within a block and probability $q$ of an edge between any two blocks. In this case, we obtain the following.
\begin{coro}\label{cor:com-hsbm-util}
    In the above HSBM, \dpcom{} recovers the exact communities when $\delta \leq \frac{1}{n}$, $k < n^{1/4}$, and $\sqrt{p} - \sqrt{q} \geq \Omega(\frac{k \ln \frac{1}{\delta}}{\sqrt{\epsilon}n^{1/4}})$.
\end{coro}
Compared to previous work in the SBM with privacy, our algorithm requires a larger assumption on $\sqrt{p} - \sqrt{q}$ (\citet{seif2022differentially,chen2023private} require $\sqrt{p} - \sqrt{q} \geq \sqrt{\frac{k}{\epsilon n}})$. However, previous work either uses semi-definite programming or does not run in polynomial time, whereas \dpcom{} is a practical use of the significantly more efficient Singular Vector Decomposition. Furthermore, our algorithm works in the fully-general HSBM, whereas previous work has no analogue of Theorem~\ref{thm:com-hsbm-util-inf}.

\begin{algorithm}
\caption{\dphchsbm{} a hierarchical clustering algorithm in the HSBM given the blocks.}\label{alg:priv-hsbm-final}
\begin{algorithmic}
\STATE{\textbf{Input:} $\hat{A}$, adjacency matrix generated from $\hsbm(B,P,f)$, number of blocks $k$, privacy parameter $\epsilon$.}
\STATE{\textbf{Output:} An hierarchical clustering $T$ of $\hat{A}$.}
\FOR{$i \in \{1, \ldots, \log n\}$}
\STATE{$\hat{F} \gets \dpcom{}(\hat{A}, \frac{\epsilon}{2 \log n})$}
\STATE{$B_1^{i}, \ldots, B_k^i \gets \textsf{k-centers}(\hat{F}, k)$}
\ENDFOR
\STATE{$B_1, \ldots, B_k \gets \textsf{Union-Find}(B_1^1, \ldots, B_k^1, \ldots, B_k^{\log n})$}
\STATE{$T \gets \dphcblocks{}(\hat{A}, \{B_1, \ldots, B_k\}, \frac{\epsilon}{2})$}
\RETURN{$T$}
\end{algorithmic}
\end{algorithm}
Combining Theorems~\ref{thm:hc-hsbm-util} and~\ref{thm:com-hsbm-util-inf}, we are able to obtain \dphchsbm{}, an end-to-end hierarchical clustering algorithm in the HSBM (Algorithm~\ref{alg:priv-hsbm-final}). This algorithm runs $\dpcom{}$ $\log n$ times, using $k$-centers each run to find the well-separated communities in the subset $Z_2 \subseteq V$ returned by $\dpcom{}$. Running $\log n$ times ensures that with high probability, each point in $V$ will participate in at least one $Z_2$; these clusters may then be merged using a union-find data structure.
\begin{coro}\label{coro:hc-hsbm}
    Let $\hat{A}$ be drawn from $\hsbm(B,P,f)$, and let $\tau = \max f(x)$ and $s = \min_{i=1}^k|B_i|$. Then, if $\epsilon > \frac{1}{\sqrt{n}}$, $\delta < \frac{1}{n}$, $s \geq n^{3/4}$, $f \geq \frac{\log n}{\sqrt{n}}$, and the parameters $s,\tau,A,\Delta$ satisfy the conditions of Theorem~\ref{thm:com-hsbm-util-inf}, then \dphchsbm{} satisfies $(\epsilon, \delta)$-edge DP and is a $1+o(1)$ approximation to the Dasgupta cost. 
\end{coro}
Corollary~\ref{coro:hc-hsbm} gives a $1+o(1)$ multiplicative approximation the the Dasgupta cost for the given parameter regimes of the HSBM. This is a nearly-optimal cost that avoids the additive error of the algorithms in Section~\ref{sec:algorithms}.

\section{Experiments} \label{sec:experiments}
The purpose of this section is evaluate Algorithm~\ref{alg:priv-hc-hsbm} designed for the HSBM model. First, we outline our methods and then we discuss our results.

\paragraph{Experimental Setup}
We tested our clustering algorithms on a real-world graph and generated synthetic graphs from the HSBM model. We compared the performance of \dphchsbm{} to several baseline algorithms.
We ran algorithms at $\epsilon \in \{0.5, 1.0, 2.0\}$, as well as with no privacy.

To enable the replication of our work, we make the code available {\bf open-source} \footnote{\url{https://bitbucket.org/jjimola/dphc/src/master/}}.

\paragraph{Datasets} Our real-world graph was generated from the MNIST digits dataset~\citep{lecun1998mnist} (with 1797 digits) by, for each digit, adding an undirected edge corresponding to one of its 120 nearest neighbors in pixel space. We generated graphs from $\hsbm(B, P, f)$ with $n=2048$ nodes, $k = \{4, 8\}$ blocks, with block sizes chosen proportional to $\{1, \gamma, \ldots, \gamma^{k-1}\}$, where $\gamma^{k-1} = 3$. This has the effect of creating differently-sized blocks. We selected $P$ to be a balanced tree over the blocks, and $f$ that increases uniformly in the interval $[0.1, 0.9]$ as the tree is descended. 

\paragraph{Algorithms} We ran \dphchsbm{} and several baseline algorithms. In the implementation of \dphchsbm{}, we used a modified version of \dpcom{} for practical considerations. This does not affect the privacy guarantees but it simplifies the algorithm. In particular, we privately release $\tilde{A}_1$ using the Laplace mechanism, and compute $\Pi_{\tilde{A}_1}(\hat{A}_2)$ without projection. We are then able to add Gaussian noise tailored to the sensitivity of $\Pi_{\tilde{A}_1}$, rather than to $\Gamma$ which proved to be a rough upper bound in practice.

For our baselines, we considered a naive private approach in which we release $A$ using the Laplace mechanism and truncate these values to be non-negative to form a sanitized, weighted graph. Then, we ran single, complete, and average linkage, and recorded the best of these methods. We refer collectively to these baselines as \linkage{}. Second, we formed a tree by recursively partitioning the graph into its (approximately) sparsest cut. As shown in~\citet{charikar2017approximate}, this is an $O(\sqrt{\log n}, 0)$-approximation in the \textit{sanitized} graph. We refer to this baseline as \sparsecut{}.

\paragraph{Metrics}
For each graph and clustering algorithm, and the value of $\epsilon$, we computed $\cost_G(T)$, averaged over 5 runs.

\subsection{Results}\label{sec:exp-discuss}
\begin{figure}
    \begin{subfigure}[b]{\linewidth}
            \centering\includegraphics[width=0.8\linewidth]{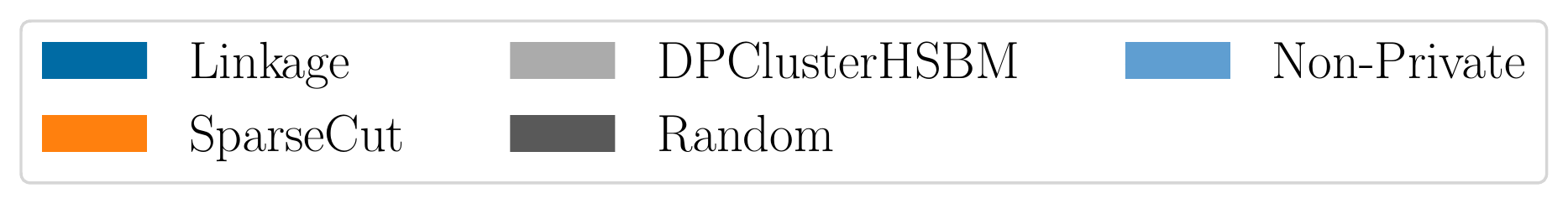}
    \end{subfigure}\\
    \begin{subfigure}[b]{\linewidth}
        \centering
        \includegraphics[width=0.48\linewidth]{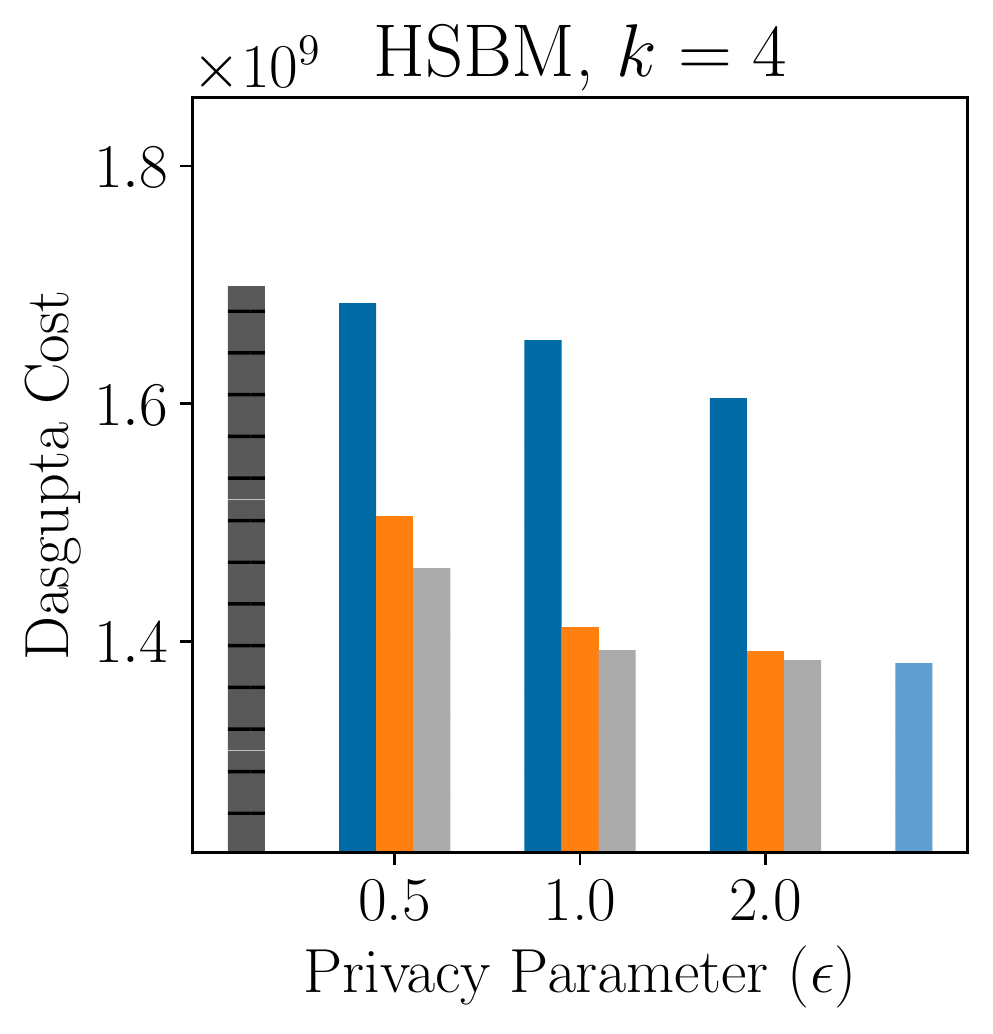}
        \includegraphics[width=0.48\linewidth]{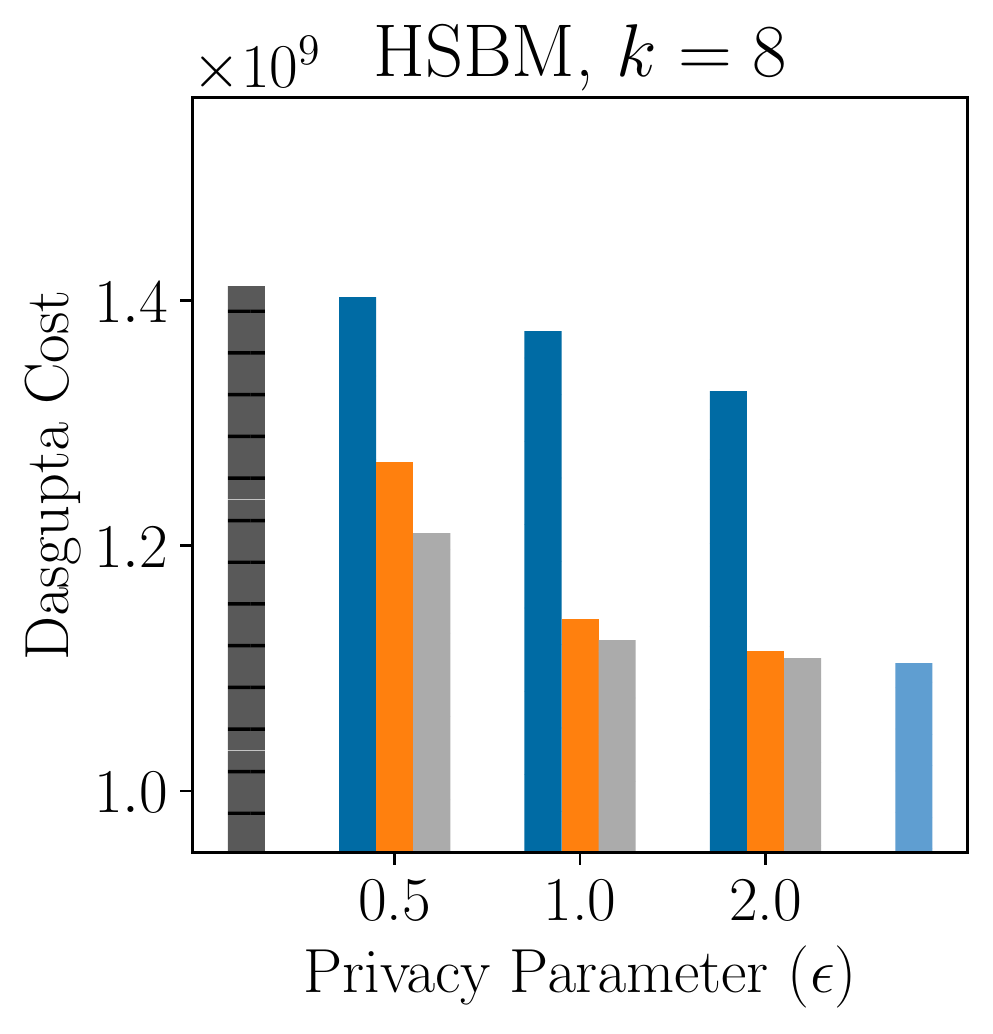}
        \includegraphics[width=0.48\linewidth]{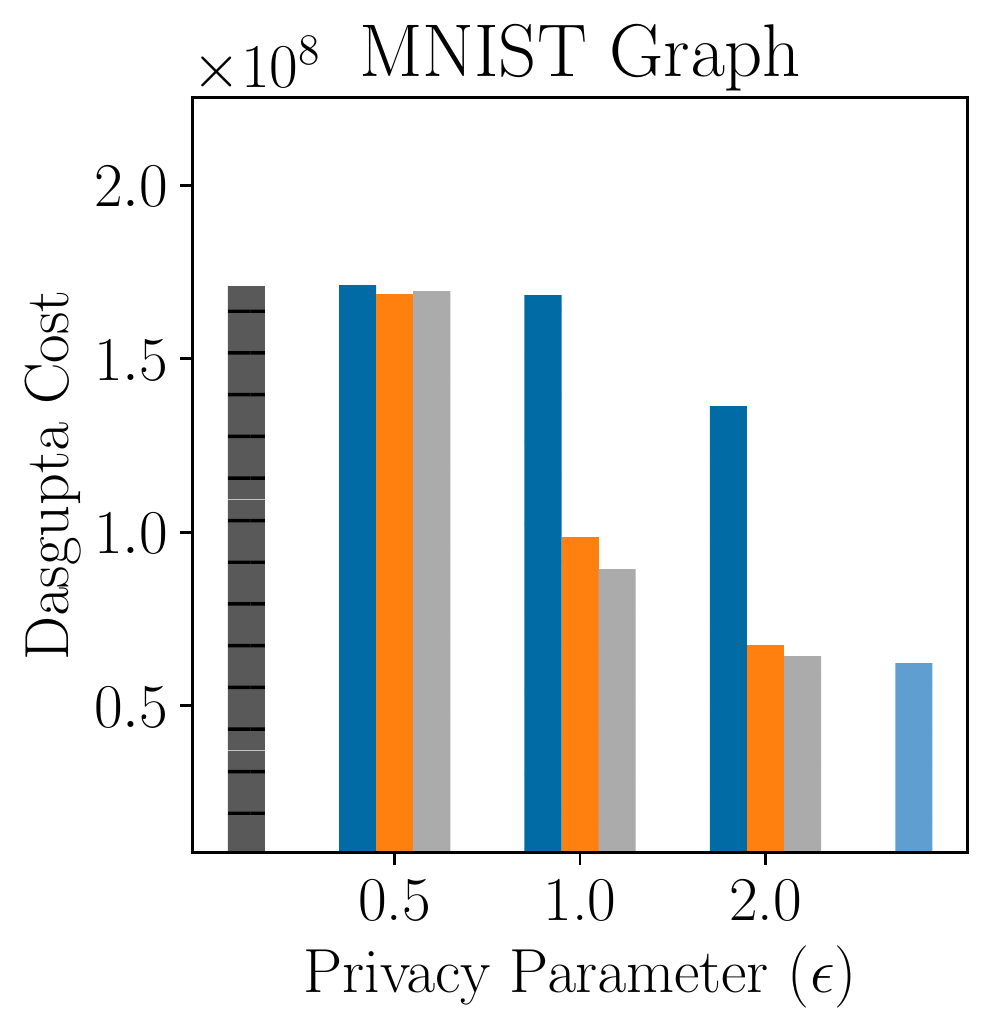}
    \end{subfigure}
    \caption{Cost for HSBM graphs with $2048$ nodes and $k$ clusters and MNIST graph with $1797$ nodes.}
    \label{fig:exp-results}
    
\end{figure}

Our results appear in Figure~\ref{fig:exp-results}. In addition to the cost for each algorithm, we included the cost of a random tree. The data had low variance: for each of the 5 runs used to compute each bar, the values were within $0.5\%$ of each other.

For all trials, 
the cost of \linkage{} was much higher than the other two algorithms; even with $\epsilon = 2$, \linkage{} did not offer improvement of more than $10\%$ reduction in cost over the random tree. Thus, the rest of our discussion focuses on \dphchsbm{} and \sparsecut{}.

For the synthetic graphs,
the cost of \dphchsbm{} is lower than \sparsecut{}, particularly when $\epsilon = 0.5$. In this case, when $k=4$ (resp. $8$), \dphchsbm{} offered a $14.4\%$ (resp. $14.2\%$) reduction in cost over the random tree, whereas \sparsecut{} offered an $11.5\%$ (resp. $10.3\%$) reduction.
Thus, \dphchsbm{} offers up to $38\%$ more reduction in cost than \sparsecut{}, over the cost of a random tree. Even when $\epsilon = 0.5$, the cost of \dphchsbm{} is just $5.8\%$ (resp. $9.6\%$) higher than the cost of the best tree with no privacy.

For $\epsilon = 1,2$ on synthetic graphs, the costs of \sparsecut{} and \dphchsbm{} fall to within $1\%$ of each other, though $\dphchsbm{}$ consistently outperforms the former for all values of $\epsilon$. Moreover, notice that for $\epsilon = 2$, the costs of both algorithms are within $1\%$ of the non-private tree, indicating that for higher $\epsilon$ the cost of privacy becomes negligible.

For the graph generated from MNIST, all algorithms perform as poorly as a random tree for $\epsilon = 0.5$. This indicates that the noise introduced by the high privacy constraint destroys the clusters, which are less-well structured than those of the HSBM graphs. At $\epsilon = 1$, the error of \sparsecut{} is $10\%$ higher than \dphchsbm{}. For $\epsilon = 2$, the cost of \sparsecut{} is $5\%$ higher than that of \dphchsbm{}, and \dphchsbm{} attains error within $3\%$ of the best tree with no privacy. This is consistent with our previous observation that \dphchsbm{} offers improvement over the baselines, particularly when $\epsilon$ is not too high.

\section{Conclusion}
We have considered hierarchical clustering under differential privacy in Dasgupta's cost framework. While strong lower bounds exist for the problem, we have proposed algorithms with nearly matching approximation guarantees. Furthermore, we showed the lower bounds can be overcome in the HSBM, and nearly optimal trees can be found in this setting using efficient methods. For future work, one could consider private hierarchical clustering in a less structured model than the HSBM in hopes of overcoming the lower bound here as well.

\bibliographystyle{plainnat}
\bibliography{citations}

\begin{thebibliography}{60}
\providecommand{\natexlab}[1]{#1}
\providecommand{\url}[1]{\texttt{#1}}
\expandafter\ifx\csname urlstyle\endcsname\relax
  \providecommand{\doi}[1]{doi: #1}\else
  \providecommand{\doi}{doi: \begingroup \urlstyle{rm}\Url}\fi

\bibitem[Agarwal et~al.(2022)Agarwal, Khanna, Li, and
  Patil]{agarwal2022sublinear}
Arpit Agarwal, Sanjeev Khanna, Huan Li, and Prathamesh Patil.
\newblock Sublinear algorithms for hierarchical clustering.
\newblock \emph{arXiv preprint arXiv:2206.07633}, 2022.

\bibitem[Arora and Upadhyay(2019)]{arora2019differentially}
Raman Arora and Jalaj Upadhyay.
\newblock On differentially private graph sparsification and applications.
\newblock \emph{Advances in neural information processing systems}, 32, 2019.

\bibitem[Balcan et~al.(2017)Balcan, Dick, Liang, Mou, and
  Zhang]{balcan2017differentially}
Maria-Florina Balcan, Travis Dick, Yingyu Liang, Wenlong Mou, and Hongyang
  Zhang.
\newblock Differentially private clustering in high-dimensional euclidean
  spaces.
\newblock In \emph{International Conference on Machine Learning}, pages
  322--331. PMLR, 2017.

\bibitem[Bateni et~al.(2017)Bateni, Behnezhad, Derakhshan, Hajiaghayi, Kiveris,
  Lattanzi, and Mirrokni]{bateni2017affinity}
Mohammadhossein Bateni, Soheil Behnezhad, Mahsa Derakhshan, MohammadTaghi
  Hajiaghayi, Raimondas Kiveris, Silvio Lattanzi, and Vahab Mirrokni.
\newblock Affinity clustering: Hierarchical clustering at scale.
\newblock In I.~Guyon, U.~V. Luxburg, S.~Bengio, H.~Wallach, R.~Fergus,
  S.~Vishwanathan, and R.~Garnett, editors, \emph{Advances in Neural
  Information Processing Systems 30}, pages 6864--6874. Curran Associates,
  Inc., 2017.

\bibitem[Bhatia(1997)]{bhatia1997}
Rajendra Bhatia.
\newblock \emph{Matrix Analysis}, volume 169.
\newblock Springer Verlag, 1997.

\bibitem[Blocki et~al.(2012)Blocki, Blum, Datta, and
  Sheffet]{blocki2012johnson}
Jeremiah Blocki, Avrim Blum, Anupam Datta, and Or~Sheffet.
\newblock The johnson-lindenstrauss transform itself preserves differential
  privacy.
\newblock In \emph{2012 IEEE 53rd Annual Symposium on Foundations of Computer
  Science}, pages 410--419. IEEE, 2012.

\bibitem[Bun et~al.(2021)Bun, Elias, and Kulkarni]{bun2021differentially}
Mark Bun, Marek Elias, and Janardhan Kulkarni.
\newblock Differentially private correlation clustering.
\newblock In \emph{International Conference on Machine Learning}, pages
  1136--1146. PMLR, 2021.

\bibitem[Charikar and Chatziafratis(2017)]{charikar2017approximate}
Moses Charikar and Vaggos Chatziafratis.
\newblock Approximate hierarchical clustering via sparsest cut and spreading
  metrics.
\newblock In \emph{Proceedings of the Twenty-Eighth Annual ACM-SIAM Symposium
  on Discrete Algorithms}, pages 841--854. SIAM, 2017.

\bibitem[Chatziafratis et~al.(2020)Chatziafratis, Yaroslavtsev, Lee,
  Makarychev, Ahmadian, Epasto, and Mahdian]{chatziafratis2020bisect}
Vaggos Chatziafratis, Grigory Yaroslavtsev, Euiwoong Lee, Konstantin
  Makarychev, Sara Ahmadian, Alessandro Epasto, and Mohammad Mahdian.
\newblock Bisect and conquer: Hierarchical clustering via max-uncut bisection.
\newblock In \emph{International Conference on Artificial Intelligence and
  Statistics}, pages 3121--3132. PMLR, 2020.

\bibitem[Chaudhuri et~al.(2011)Chaudhuri, Monteleoni, and
  Sarwate]{chaudhuri2011differentially}
Kamalika Chaudhuri, Claire Monteleoni, and Anand~D Sarwate.
\newblock Differentially private empirical risk minimization.
\newblock \emph{Journal of Machine Learning Research}, 12\penalty0 (3), 2011.

\bibitem[Chen et~al.(2023)Chen, Cohen-Addad, d'Orsi, Epasto, Imola, Steurer,
  and Tiegel]{chen2023private}
Hongjie Chen, Vincent Cohen-Addad, Tommaso d'Orsi, Alessandro Epasto, Jacob
  Imola, David Steurer, and Stefan Tiegel.
\newblock Private estimation algorithms for stochastic block models and mixture
  models.
\newblock \emph{arXiv preprint arXiv:2301.04822}, 2023.

\bibitem[Cohen-Addad et~al.(2017)Cohen-Addad, Kanade, and
  Mallmann-Trenn]{cohen2017hierarchical}
Vincent Cohen-Addad, Varun Kanade, and Frederik Mallmann-Trenn.
\newblock Hierarchical clustering beyond the worst-case.
\newblock \emph{Advances in Neural Information Processing Systems}, 30, 2017.

\bibitem[Cohen-Addad et~al.(2019)Cohen-Addad, Kanade, Mallmann-Trenn, and
  Mathieu]{cohen2019hierarchical}
Vincent Cohen-Addad, Varun Kanade, Frederik Mallmann-Trenn, and Claire Mathieu.
\newblock Hierarchical clustering: Objective functions and algorithms.
\newblock \emph{Journal of the ACM (JACM)}, 66\penalty0 (4):\penalty0 1--42,
  2019.

\bibitem[Cohen-Addad et~al.(2022{\natexlab{a}})Cohen-Addad, Epasto, Lattanzi,
  Mirrokni, Munoz, Saulpic, Schwiegelshohn, and
  Vassilvitskii]{cohen2022scalable}
Vincent Cohen-Addad, Alessandro Epasto, Silvio Lattanzi, Vahab Mirrokni, Andres
  Munoz, David Saulpic, Chris Schwiegelshohn, and Sergei Vassilvitskii.
\newblock Scalable differentially private clustering via hierarchically
  separated trees.
\newblock \emph{arXiv preprint arXiv:2206.08646}, 2022{\natexlab{a}}.

\bibitem[Cohen-Addad et~al.(2022{\natexlab{b}})Cohen-Addad, Epasto, Mirrokni,
  Narayanan, and Zhong]{cohennear}
Vincent Cohen-Addad, Alessandro Epasto, Vahab Mirrokni, Shyam Narayanan, and
  Peilin Zhong.
\newblock Near-optimal private and scalable $ k $-clustering.
\newblock In \emph{Advances in Neural Information Processing Systems},
  2022{\natexlab{b}}.

\bibitem[Cohen-Addad et~al.(2022{\natexlab{c}})Cohen-Addad, Fan, Lattanzi,
  Mitrovi{\'c}, Norouzi-Fard, Parotsidis, and Tarnawski]{cohen2022near}
Vincent Cohen-Addad, Chenglin Fan, Silvio Lattanzi, Slobodan Mitrovi{\'c},
  Ashkan Norouzi-Fard, Nikos Parotsidis, and Jakub Tarnawski.
\newblock Near-optimal correlation clustering with privacy.
\newblock \emph{arXiv preprint arXiv:2203.01440}, 2022{\natexlab{c}}.

\bibitem[Dasgupta(2016)]{dasgupta2016cost}
Sanjoy Dasgupta.
\newblock A cost function for similarity-based hierarchical clustering.
\newblock In \emph{Proceedings of the forty-eighth annual ACM symposium on
  Theory of Computing}, pages 118--127, 2016.

\bibitem[Dhulipala et~al.(2022)Dhulipala, Eisenstat, {\L}acki, Mirronki, and
  Shi]{dhulipala2022hierarchical}
Laxman Dhulipala, David Eisenstat, Jakub {\L}acki, Vahab Mirronki, and Jessica
  Shi.
\newblock Hierarchical agglomerative graph clustering in poly-logarithmic
  depth.
\newblock In \emph{Neurips 2022}, 2022.

\bibitem[Diez et~al.(2015)Diez, Bonifazi, Escudero, Mateos, Mu{\~n}oz,
  Stramaglia, and Cortes]{diez2015novel}
Ibai Diez, Paolo Bonifazi, I{\~n}aki Escudero, Beatriz Mateos, Miguel~A
  Mu{\~n}oz, Sebastiano Stramaglia, and Jesus~M Cortes.
\newblock A novel brain partition highlights the modular skeleton shared by
  structure and function.
\newblock \emph{Scientific reports}, 5:\penalty0 10532, 2015.

\bibitem[Ding et~al.(2022)Ding, d'Orsi, Nasser, and Steurer]{ding2022robust}
Jingqiu Ding, Tommaso d'Orsi, Rajai Nasser, and David Steurer.
\newblock Robust recovery for stochastic block models.
\newblock In \emph{2021 IEEE 62nd Annual Symposium on Foundations of Computer
  Science (FOCS)}, pages 387--394. IEEE, 2022.

\bibitem[Dwork(2019)]{dwork2019differential}
Cynthia Dwork.
\newblock Differential privacy and the us census.
\newblock In \emph{Proceedings of the 38th ACM SIGMOD-SIGACT-SIGAI symposium on
  principles of database systems}, pages 1--1, 2019.

\bibitem[Dwork et~al.(2006)Dwork, McSherry, Nissim, and
  Smith]{dwork2006calibrating}
Cynthia Dwork, Frank McSherry, Kobbi Nissim, and Adam Smith.
\newblock Calibrating noise to sensitivity in private data analysis.
\newblock In \emph{Theory of cryptography conference}, pages 265--284.
  Springer, 2006.

\bibitem[Dwork et~al.(2014{\natexlab{a}})Dwork, Roth,
  et~al.]{dwork2014algorithmic}
Cynthia Dwork, Aaron Roth, et~al.
\newblock The algorithmic foundations of differential privacy.
\newblock \emph{Foundations and Trends{\textregistered} in Theoretical Computer
  Science}, 9\penalty0 (3--4):\penalty0 211--407, 2014{\natexlab{a}}.

\bibitem[Dwork et~al.(2014{\natexlab{b}})Dwork, Talwar, Thakurta, and
  Zhang]{dwork2014analyze}
Cynthia Dwork, Kunal Talwar, Abhradeep Thakurta, and Li~Zhang.
\newblock Analyze gauss: optimal bounds for privacy-preserving principal
  component analysis.
\newblock In \emph{Proceedings of the forty-sixth annual ACM symposium on
  Theory of computing}, pages 11--20, 2014{\natexlab{b}}.

\bibitem[Eisen et~al.(1998)Eisen, Spellman, Brown, and
  Botstein]{eisen1998cluster}
Michael~B Eisen, Paul~T Spellman, Patrick~O Brown, and David Botstein.
\newblock Cluster analysis and display of genome-wide expression patterns.
\newblock \emph{Proceedings of the National Academy of Sciences}, 95\penalty0
  (25):\penalty0 14863--14868, 1998.

\bibitem[Eli{\'a}{\v{s}} et~al.(2020)Eli{\'a}{\v{s}}, Kapralov, Kulkarni, and
  Lee]{eliavs2020differentially}
Marek Eli{\'a}{\v{s}}, Michael Kapralov, Janardhan Kulkarni, and Yin~Tat Lee.
\newblock Differentially private release of synthetic graphs.
\newblock In \emph{Proceedings of the Fourteenth Annual ACM-SIAM Symposium on
  Discrete Algorithms}, pages 560--578. SIAM, 2020.

\bibitem[Epasto et~al.(2022)Epasto, Mirrokni, Perozzi, Tsitsulin, and
  Zhong]{epasto2022differentially}
Alessandro Epasto, Vahab Mirrokni, Bryan Perozzi, Anton Tsitsulin, and Peilin
  Zhong.
\newblock Differentially private graph learning via sensitivity-bounded
  personalized pagerank.
\newblock In \emph{Neurips}, 2022.

\bibitem[Erlingsson et~al.(2014)Erlingsson, Pihur, and
  Korolova]{erlingsson2014rappor}
{\'U}lfar Erlingsson, Vasyl Pihur, and Aleksandra Korolova.
\newblock Rappor: Randomized aggregatable privacy-preserving ordinal response.
\newblock In \emph{Proceedings of the 2014 ACM SIGSAC conference on computer
  and communications security}, pages 1054--1067, 2014.

\bibitem[Fei and Chen(2020)]{MR4115142}
Yingjie Fei and Yudong Chen.
\newblock Achieving the {B}ayes error rate in synchronization and block models
  by {SDP}, robustly.
\newblock \emph{IEEE Trans. Inform. Theory}, 66\penalty0 (6):\penalty0
  3929--3953, 2020.
\newblock ISSN 0018-9448.
\newblock \doi{10.1109/TIT.2020.2966438}.
\newblock URL \url{https://doi.org/10.1109/TIT.2020.2966438}.

\bibitem[Ghazi et~al.(2020)Ghazi, Kumar, and
  Manurangsi]{ghazi2020differentially}
Badih Ghazi, Ravi Kumar, and Pasin Manurangsi.
\newblock Differentially private clustering: Tight approximation ratios.
\newblock \emph{Advances in Neural Information Processing Systems},
  33:\penalty0 4040--4054, 2020.

\bibitem[Gu\'edon and Vershynin(2016)]{MR3520025-Guedon16}
Olivier Gu\'edon and Roman Vershynin.
\newblock Community detection in sparse networks via {G}rothendieck's
  inequality.
\newblock \emph{Probab. Theory Related Fields}, 165\penalty0 (3-4):\penalty0
  1025--1049, 2016.
\newblock ISSN 0178-8051.
\newblock \doi{10.1007/s00440-015-0659-z}.
\newblock URL \url{http://dx.doi.org/10.1007/s00440-015-0659-z}.

\bibitem[Hardt and Talwar(2010)]{hardt2010geometry}
Moritz Hardt and Kunal Talwar.
\newblock On the geometry of differential privacy.
\newblock In \emph{Proceedings of the forty-second ACM symposium on Theory of
  computing}, pages 705--714, 2010.

\bibitem[Hegde et~al.(2021)Hegde, M{\"o}llering, Schneider, and
  Yalame]{hegde2021sok}
Aditya Hegde, Helen M{\"o}llering, Thomas Schneider, and Hossein Yalame.
\newblock Sok: Efficient privacy-preserving clustering.
\newblock \emph{Proceedings on Privacy Enhancing Technologies}, 2021\penalty0
  (4):\penalty0 225--248, 2021.

\bibitem[Jain(2010)]{J10}
Anil~K. Jain.
\newblock Data clustering: 50 years beyond k-means.
\newblock \emph{Pattern Recognition Letters}, 31\penalty0 (8):\penalty0
  651--666, 2010.
\newblock \doi{10.1016/j.patrec.2009.09.011}.
\newblock URL \url{https://doi.org/10.1016/j.patrec.2009.09.011}.

\bibitem[Jardine and Sibson(1968)]{jardine1968model}
N~Jardine and R~Sibson.
\newblock A model for taxonomy.
\newblock \emph{Mathematical Biosciences}, 2\penalty0 (3-4):\penalty0 465--482,
  1968.

\bibitem[Johnson(1984)]{johnson1984extensions}
William~B Johnson.
\newblock Extensions of lipschitz mappings into a hilbert space.
\newblock \emph{Contemp. Math.}, 26:\penalty0 189--206, 1984.

\bibitem[Kasiviswanathan et~al.(2013)Kasiviswanathan, Nissim, Raskhodnikova,
  and Smith]{kasiviswanathan2013analyzing}
Shiva~Prasad Kasiviswanathan, Kobbi Nissim, Sofya Raskhodnikova, and Adam
  Smith.
\newblock Analyzing graphs with node differential privacy.
\newblock In \emph{Theory of Cryptography Conference}, pages 457--476.
  Springer, 2013.

\bibitem[Kolluri et~al.(2021)Kolluri, Baluta, and Saxena]{kolluri2021private}
Aashish Kolluri, Teodora Baluta, and Prateek Saxena.
\newblock Private hierarchical clustering in federated networks.
\newblock In \emph{Proceedings of the 2021 ACM SIGSAC Conference on Computer
  and Communications Security}, pages 2342--2360, 2021.

\bibitem[LeCun(1998)]{lecun1998mnist}
Yann LeCun.
\newblock The mnist database of handwritten digits.
\newblock \emph{http://yann. lecun. com/exdb/mnist/}, 1998.

\bibitem[Leskovec et~al.(2014)Leskovec, Rajaraman, and
  Ullman]{leskovec2014mining}
Jure Leskovec, Anand Rajaraman, and Jeffrey~David Ullman.
\newblock \emph{Mining of massive datasets}.
\newblock Cambridge university press, 2014.

\bibitem[Liu and Moitra(2022)]{Liu-Moitra-minimax}
Allen Liu and Ankur Moitra.
\newblock Minimax rates for robust community detection.
\newblock \emph{CoRR}, abs/2207.11903, 2022.
\newblock \doi{10.48550/arXiv.2207.11903}.
\newblock URL \url{https://doi.org/10.48550/arXiv.2207.11903}.

\bibitem[Machanavajjhala et~al.(2017)Machanavajjhala, He, and
  Hay]{machanavajjhala2017differential}
Ashwin Machanavajjhala, Xi~He, and Michael Hay.
\newblock Differential privacy in the wild: A tutorial on current practices \&
  open challenges.
\newblock In \emph{Proceedings of the 2017 ACM International Conference on
  Management of Data}, pages 1727--1730, 2017.

\bibitem[Mann et~al.(2008)Mann, Matula, and Olinick]{mann2008use}
Charles~F Mann, David~W Matula, and Eli~V Olinick.
\newblock The use of sparsest cuts to reveal the hierarchical community
  structure of social networks.
\newblock \emph{Social Networks}, 30\penalty0 (3):\penalty0 223--234, 2008.

\bibitem[McSherry(2001)]{mcsherry2001spectral}
Frank McSherry.
\newblock Spectral partitioning of random graphs.
\newblock In \emph{Proceedings 42nd IEEE Symposium on Foundations of Computer
  Science}, pages 529--537. IEEE, 2001.

\bibitem[McSherry and Talwar(2007)]{mcsherry2007mechanism}
Frank McSherry and Kunal Talwar.
\newblock Mechanism design via differential privacy.
\newblock In \emph{48th Annual IEEE Symposium on Foundations of Computer
  Science (FOCS'07)}, pages 94--103. IEEE, 2007.

\bibitem[Moitra et~al.(2016)Moitra, Perry, and Wein]{moitra2016robust}
Ankur Moitra, William Perry, and Alexander~S Wein.
\newblock How robust are reconstruction thresholds for community detection?
\newblock In \emph{Proceedings of the forty-eighth annual ACM symposium on
  Theory of Computing}, pages 828--841, 2016.

\bibitem[Montanari and Sen(2016)]{montanari2016semidefinite}
Andrea Montanari and Subhabrata Sen.
\newblock Semidefinite programs on sparse random graphs and their application
  to community detection.
\newblock In \emph{Proceedings of the forty-eighth annual ACM symposium on
  Theory of Computing}, pages 814--827, 2016.

\bibitem[Moseley and Wang(2017)]{moseley2017approximation}
Benjamin Moseley and Joshua Wang.
\newblock Approximation bounds for hierarchical clustering: Average linkage,
  bisecting k-means, and local search.
\newblock \emph{Advances in neural information processing systems}, 30, 2017.

\bibitem[Murtagh and Contreras(2012)]{murtagh2012algorithms}
Fionn Murtagh and Pedro Contreras.
\newblock Algorithms for hierarchical clustering: an overview.
\newblock \emph{Wiley Interdisciplinary Reviews: Data Mining and Knowledge
  Discovery}, 2\penalty0 (1):\penalty0 86--97, 2012.

\bibitem[Pinot(2018)]{pinot2018minimum}
Rafael Pinot.
\newblock Minimum spanning tree release under differential privacy constraints.
\newblock \emph{arXiv preprint arXiv:1801.06423}, 2018.

\bibitem[Roy~Chowdhury et~al.(2020)Roy~Chowdhury, Wang, He, Machanavajjhala,
  and Jha]{roy2020crypte}
Amrita Roy~Chowdhury, Chenghong Wang, Xi~He, Ashwin Machanavajjhala, and Somesh
  Jha.
\newblock Crypte: Crypto-assisted differential privacy on untrusted servers.
\newblock In \emph{Proceedings of the 2020 ACM SIGMOD International Conference
  on Management of Data}, pages 603--619, 2020.

\bibitem[Seif et~al.(2022)Seif, Nguyen, Vullikanti, and
  Tandon]{seif2022differentially}
Mohamed Seif, Dung Nguyen, Anil Vullikanti, and Ravi Tandon.
\newblock Differentially private community detection for stochastic block
  models.
\newblock \emph{arXiv preprint arXiv:2202.00636}, 2022.

\bibitem[Sneath and Sokal(1962)]{sneath1962numerical}
Peter~HA Sneath and Robert~R Sokal.
\newblock Numerical taxonomy.
\newblock \emph{Nature}, 193\penalty0 (4818):\penalty0 855--860, 1962.

\bibitem[Steinbach et~al.(2000)Steinbach, Karypis, Kumar,
  et~al.]{steinbach2000}
Michael Steinbach, George Karypis, Vipin Kumar, et~al.
\newblock A comparison of document clustering techniques.
\newblock In \emph{KDD workshop on text mining}, volume 400, pages 525--526.
  Boston, 2000.

\bibitem[Tumminello et~al.(2010)Tumminello, Lillo, and
  Mantegna]{tumminello2010correlation}
Michele Tumminello, Fabrizio Lillo, and Rosario~N Mantegna.
\newblock Correlation, hierarchies, and networks in financial markets.
\newblock \emph{Journal of Economic Behavior \& Organization}, 75\penalty0
  (1):\penalty0 40--58, 2010.

\bibitem[Virtanen et~al.(2020)Virtanen, Gommers, Oliphant, Haberland, Reddy,
  Cournapeau, Burovski, Peterson, Weckesser, Bright, et~al.]{virtanen2020scipy}
Pauli Virtanen, Ralf Gommers, Travis~E Oliphant, Matt Haberland, Tyler Reddy,
  David Cournapeau, Evgeni Burovski, Pearu Peterson, Warren Weckesser, Jonathan
  Bright, et~al.
\newblock Scipy 1.0: fundamental algorithms for scientific computing in python.
\newblock \emph{Nature methods}, 17\penalty0 (3):\penalty0 261--272, 2020.

\bibitem[Vu(2014)]{vu2014simple}
Van Vu.
\newblock A simple svd algorithm for finding hidden partitions.
\newblock \emph{arXiv preprint arXiv:1404.3918}, 2014.

\bibitem[Vu(2005)]{vu2005spectral}
Van~H Vu.
\newblock Spectral norm of random matrices.
\newblock In \emph{Proceedings of the thirty-seventh annual ACM symposium on
  Theory of computing}, pages 423--430, 2005.

\bibitem[Ward~Jr(1963)]{ward1963hierarchical}
Joe~H Ward~Jr.
\newblock Hierarchical grouping to optimize an objective function.
\newblock \emph{Journal of the American statistical association}, 58\penalty0
  (301):\penalty0 236--244, 1963.

\bibitem[Xiao et~al.(2014)Xiao, Chen, and Tan]{xiao2014differentially}
Qian Xiao, Rui Chen, and Kian-Lee Tan.
\newblock Differentially private network data release via structural inference.
\newblock In \emph{Proceedings of the 20th ACM SIGKDD international conference
  on Knowledge discovery and data mining}, pages 911--920, 2014.

\end{thebibliography}
\appendix
\onecolumn
\section{Related Work}\label{app:related}

\paragraph{Differential Privacy} 
Differential privacy~\citep{dwork2006calibrating} has recently become the gold standard of privacy used by institutions such as the US census~\citep{dwork2019differential} and large tech companies~\citep{erlingsson2014rappor}. In a nutshell, DP algorithms provide plausible deniability for the input data of any user.   
There is a vast literature on DP algorithms for a disparate range of problems and many different models for differential privacy~\citep{dwork2006calibrating,mcsherry2007mechanism,chaudhuri2011differentially,roy2020crypte,machanavajjhala2017differential,dwork2019differential} (we refer to~\citet{dwork2014algorithmic} for a survey). 

Among this rapidly growing literature, our work builds on multiple work on differentially privacy, namely DP PCA algorithms~\citep{dwork2014analyze}, DP Johnson Lindenstrauss projections~ \citep{blocki2012johnson}, DP cut sparsification in graphs~\citep{eliavs2020differentially} as well as DP stochastic block model reconstruction (reviewed later). 

\paragraph{Private graph algorithms}
Especially relevant to this work is the area of differential privacy in graphs. DP has been  declined in graph problems both as  the edge-level~\citep{epasto2022differentially,eliavs2020differentially} and node-level  model~\citep{kasiviswanathan2013analyzing}. The most related work in this area is that on graph cut approximation~\citep{eliavs2020differentially,arora2019differentially}, as well as that of graph clustering with DP in correlation clustering model~\citep{bun2021differentially, cohen2022near}. 

\paragraph{Hierarchical Clustering}
As we discussed in the introduction, hierarchical clustering has been studied for decades in multiple fields. For this reason, a significant number of algorithms for hierarchical clustering have been introduced~\citep{murtagh2012algorithms}. Up until recently~\citep{dasgupta2016cost}, most work on hierarchical clustering has been heuristic in nature, defining algorithms based on procedures without specific theoretical guarantees in terms of approximation. Most well-known among such algorithms are the linkage-based ones~\citep{J10,bateni2017affinity}. \citet{dasgupta2016cost} introduced for the first time a combinatorial approximation objective for hierarchical clustering which is the one studied in this paper. Since this work, many authors have designed algorithms for variants of the problem~\citep{cohen2017hierarchical,cohen2019hierarchical, charikar2017approximate,moseley2017approximation, agarwal2022sublinear,chatziafratis2020bisect} exploring maximization/minimization versions of the problem on dissimilarity/similarity graphs.

Limited work has been devoted to DP hierarchical clustering algorithms. One paper~\citep{xiao2014differentially} initiates private clustering via MCMC methods, which are not guaranteed to be polynomial time. Follow-up work~\citep{kolluri2021private} shows that sampling from the Boltzmann distribution (essentially the exponential mechanism~\citep{mcsherry2007mechanism} in DP) produces an approximation to the maximization version of Dasgupta's function, which is a different problem formulation. Again, this algorithm is not provably polynomial time.

\paragraph{Private flat clustering}
Contrary to hierarchical clustering, the area of private {\it flat} clustering on metric spaces has received large attention. Most work in this area has focus on improving the privacy-approximation trade-off~\citep{ghazi2020differentially,balcan2017differentially} and on efficiency~\citep{hegde2021sok,cohennear,cohen2022scalable}.

\paragraph{Stochastic block models}

The Stochastic Block Model (SBM) is a classic model for random graphs with planted partitions which has received  significant attention in the literature. Most work in this area has focus on providing exact or approximate recovery of communities for increasingly more difficult regimes of the model~\citep{MR3520025-Guedon16,montanari2016semidefinite, moitra2016robust,MR4115142,ding2022robust,Liu-Moitra-minimax}. Specifically for our work, we focus on a variant of the model which has nested ground-truth communities arranged in a hierarchical fashion. This model has received attention for hierarchical clustering~\citep{cohen2017hierarchical}.   

The study of private algorithms for SBMs is instead very recent and no work has addressed private recovery for hierarchical SBMS. One of the only results known for private (non-hierarchical) SBMs is the work of~\citet{seif2022differentially} which provides a quasi-polynomial time algorithm for some regimes of the model. This paper require either non-poly time or $\epsilon \in \Omega(\log(|V|))$. 
Finally, very recently and currently to our work, the manuscript of~\citet{chen2023private} has been published. This work provides strong approximation guarantees using semi-definite programming for recovering SBM communities.    
None of these papers can be used directly to approximate hierarchical clustering on HSBMs. For this reason in Section~\ref{sec:algorithms-hsbm} we design a hierarchical clustering algorithm (Algorithm~\ref{alg:priv-hc-hsbm}) which uses as subroutine a DP SBM community detection algorithm. Moreover, we show a novel algorithm for SBMs (Algorithm~\ref{alg:priv-hsbm}) (independent to that of~\citet{chen2023private}) which is of practical interest as it does not require procedure with large polynomial dependency on the size of the input, such  solving a complex semi-definite program.

\section{Omitted proofs from Section~\ref{sec:lower-bounds}}\label{app:lower-bounds}
\subsection{Proof of Lemma~\ref{lem:packing-2}}

We start with the following lemma:

\begin{lem}\label{lem:two-packing}
    Let $G_1, G_2$ be two graphs drawn uniformly at random from $\calP(n,5)$. Let $\alpha = \frac{1}{100}$. The probability that there exists a balanced cut $(A,B)$ which misses at most $ \frac{\alpha}{5}n$ of the cycles for both $G_1, G_2$ is at most $2^{-0.4n}$.
\end{lem}
\begin{proof}
Let $(A,B)$ be any balanced cut with $|A| = \beta n$, for $\frac{1}{3} \leq \beta \leq \frac{2}{3}$. Let $\calE_1(A,B)$ be the event that $(A,B)$ misses at most $\frac{\alpha}{X}n$ cycles in $G_1$, and define $\calE_2(A,B)$ similarly for $G_2$. We observe the desired probability can be upper bounded by
\begin{equation}\label{eq:total-prob}
    \sum_{\substack{(A,B) \text{ a balanced cut}}} \Pr[\calE_1(A,B)] \Pr[\calE_2(A,B)].
\end{equation}

In the above sum, the balanced cuts $(A,B)$ are fixed, and the graphs $G_1, G_2$ are generated independently. We consider an equivalent random process, where $G_1 \in \calP(n,5)$ is fixed, and then $(A,B)$ is generated by picking a uniformly random string $S \in \{0,1\}^n$ with $\beta n$ $1$s. There are $\binom{n}{\beta n}$ possible strings. We will now upper bound the number of strings for which $\calE_1(A,B)$ holds. When $\calE_1(A,B)$ holds, we can choose $c$ cycles which are monochromatic $1$s, where $c$ is a non-negative integer such that $5c < n$, plus $\frac{\alpha n}{5}$ cycles which are not necessarily monochromatic. Within these $\frac{\alpha n}{5}$ cycles, there are $\alpha n$ vertices from which we can choose $d \leq \alpha n$ remaining $1$s. The total number of $1$s is $5c + d$, and thus $5c + d = \beta n$. Thus, the total number of admissible strings is at most 
\[
    \sum_{5c + d = \beta n, d \leq \alpha n}\binom{n/5}{c} \binom{n/5}{\alpha n / 5} \binom{\alpha n}{d}.
\]
We make the simple observation that $\binom{\alpha n}{d} \leq 2^{\alpha n}$. Furthermore, we observe that there are $\frac{\alpha n}{5}$ admissible choices of $c,d$. In the following, we use the fact that $2^{H_2(\beta) n - \ln n} \leq \binom{n}{\beta n} \leq 2^{H_2(\beta) n}$, where $H_2(p)$ is the binary entropy function. We upper bound the number of admissible strings with
\begin{align*}
    \frac{\alpha n}{5} \max_{(\beta-\alpha) n \leq 5c \leq \beta n} \binom{n/5}{c} \binom{n/5}{\alpha n / 5} 2^{\alpha n} &\leq 
    \frac{\alpha n}{5} \max_{(\beta - \alpha) n \leq 5c \leq \beta n} 2^{H_2(5c/n)n/5} 2^{H_2(\alpha)n/5}2^{\alpha n} \\ 
    &\leq n 2^{H_2(\beta)n/5} 2^{H_2(\alpha)n/5} 2^{\alpha n}.
\end{align*}
Dividing this number by $\binom{n}{\beta n}$, the total possible number of strings, we obtain
\begin{align*}
    \Pr[\calE_1(A,B)] &\leq \frac{n 2^{(H_2(\beta) + H_2(\alpha)) n / 5 + \alpha n}}{2^{H_2(\beta)n - \ln n}} \\
    &\leq 2^{\left(\frac{H_2(\beta) + H_2(\alpha)}{5} + \alpha - H_2(\beta) \right) n  + \ln n } \\
    &\leq 2^{-0.7 n},
\end{align*}
where the last line follows from the fact that $\frac{1}{3} \leq \beta \leq \frac{2}{3}$ and that $\alpha = \frac{1}{100}$ so that $H_2(\alpha) \leq 0.081$.
By a similar argument, we have $\Pr[A_2(B)] \leq 2^{-0.7 n}$.

Thus,~\eqref{eq:total-prob} can be upper bounded by
\[
    2^{n} \Pr[\calE_1(A,B)] \Pr[\calE_2(A,B)] \leq 2^{n} 2^{-2 \times 0.7 n} \leq 2^{-0.4n}.
\]

\end{proof}

Having shown the result for two random graphs, we apply the union bound to show that for exponentially many random graphs, it is unlikely that any tree can cluster more than one graph in the family well. We now prove Lemma~\ref{lem:packing-2}.

\begin{proof}
    Let $\calF$ consist of $2^{0.2n}$ graphs generated uniformly at random $\calP(n,5)$. For each pair of graphs $G_1, G_2$, we have by Lemma~\ref{lem:two-packing} every balanced cut will miss at least $\frac{\alpha}{5} n$ cycles in either $G_1$ or $G_2$ with probability $1-2^{-0.4n}$. By the union bound applied $\frac{1}{2}2^{0.4n}$ times for each pair of graphs, we have with probability $\frac{1}{2}$ that every balanced cut will miss at least $\frac{\alpha}{5} n$ cycles in all but at most one graph in $\calF$.
    
    Every tree can be mapped to a balanced cut, so by Lemma~\ref{lem:clique-miss-bad}, any tree will cost at least $\frac{4\alpha}{15}n^2 \geq \frac{n^2}{400}$ on all but at most one member of $\calF$. This allows us to conclude that the sets $\calB(G,r)$ are disjoint for all $G \in \calF$.
\end{proof}

\section{Omitted proofs from Section~\ref{sec:algorithms}}

\subsection{Proof of Theorem~\ref{thm:dp-spectral-sparse}}
First, we state a theorem about private graph sparsification.
\begin{thm}\label{thm:dp-spectral-sparse}
    There is a polynomial-time, $(\epsilon, \delta)$-edge differentially private algorithm which, on input graph $G = (V, E, w)$, outputs a graph $G'$ which with probability $0.9$ is a $(z, O(nz))$-approximation to cut queries in $G$, where $z = O(\frac{\log^2 \frac{1}{\delta}}{\epsilon}\frac{\log n}{\sqrt{n}})$.
\end{thm}
\begin{proof}
We apply an edge sparsification algorithm of~\citet{arora2019differentially}, which given a graph with Laplacian $L$, outputs a graph with Laplacian $L'$ with $O(\frac{n}{\gamma^2})$ edges such that
\[
    (1-\gamma) ((1-z)L + z L_n) \preceq L' \preceq (1+\gamma) ((1-z)L + z L_n),
\]
where $L_n$ is the Laplacian of an unweighted $K_n$. The value of the cut $w(S, \overline{S})$ is given by by $\textbf{1}_S^T L \textbf{1}_S$; therefore, we have
\begin{align*}
    (1-\gamma) ((1-z) w(S, \overline(S)) - z |S|(n-|S|)) &\leq w'(S, \overline{S}) \leq (1+\gamma) ((1-z) w(S, \overline{S}) + z |S|(n-|S|))
\end{align*}

Using the fact that $|S|(n-|S|) \leq n \min \{|S|, n-|S|\}$ and letting $\gamma \rightarrow 0$, we estabish that $G'$ is a $(z, n z)$ approximation to cut queries in $G$.
\end{proof}
Next, we reduce the cost to a sum of cuts. This idea appeared in~\citet{agarwal2022sublinear}. 
\begin{lem}\label{lem:hc-cuts}
    Suppose $G'$ is an $(\alpha_n, \beta_n)$-approximation to cut queries in $G$ for some $\alpha < 1$. Let $T'$ be any tree which satisfies $\cost_{G'}(T') \leq a_n \cost_{G'}^*$. Then,
    \[
        \cost_G(T') \leq (1+2\alpha_n) a_n \cost_G^* + (4a_n + 2) \beta_n n^2.
    \]
    For the revenue objective, let $T'$ be any tree which satisfies $\cost_{G'}^{\mw}(T') \geq a_n \cost_{G'}^{\mw*}$. Then,
    \[
        \cost_G^{\mw}(T') \geq (1-2\alpha_n) a_n \cost_{G}^{\mw*} - 2(a_n+1)\beta_n n^2 - 2(a_n+1)\alpha_n n^3.
    \]
\end{lem}
A proof of this lemma appears in the next section.

Finally, we are ready to prove the theorem.
\begin{proof} (Of Theorem~\ref{thm:dp-spectral-sparse}): 
    First, release a private graph $G'$ using Theorem~\ref{thm:dp-spectral-sparse}, which is a $(z, nz)$-cut approximation with probability at least $0.9$, where $z = O(\frac{\log^2 \frac{1}{\delta}}{\epsilon}\frac{\log n}{\sqrt{n}})$. We use the black box hierarchical clustering algorithm, which finds a tree such that $\E[\cost_G(T')] \leq a_n \cost_G^*$. Then, we apply Lemma~\ref{lem:hc-cuts}, obtaining
    \[
        \E[\cost_G(T')] \leq (1+2z) a_n \cost_G^* + (4a_n + 2) z n^3.
    \]
    For the revenue objective, our black box hierarchical clustering finds a tree $T'$ such that $\E[\cost_G^{\mw}(G')] \geq a_n \cost_G^{\mw*}$. We apply Lemma~\ref{lem:hc-cuts}, obtaining
    \[
        \cost_G^{\mw}(T') \geq (1-2z) a_n \cost_{G}^{\mw*} - 4(a_n+1)z n^3.
    \]
\end{proof}

\subsection{Proof of Lemma~\ref{lem:hc-cuts}}\label{app:hc-cuts}
    We start with the well-known representation of $\cost_G(T)$~\citep{dasgupta2016cost}:
    \[
        \cost_G(T) = \sum_{S \rightarrow (S_1, S_2) \text{ in $T$}} |S|w(S_1, S_2),
    \]
    where the sum is indexed by internal splits of $T$, which splits a set $S$ of leaves into two parts $S_1, S_2$. Using the identity $w(S_1, S_2) = \frac{1}{2}w(S_1, \overline{S_1}) + \frac{1}{2}w(S_2, \overline{S_2}) - \frac{1}{2}w(S, \overline{S})$, we substitute:
    \begin{align*}
        \cost_G(T) &= \frac{1}{2}\sum_{S \rightarrow (S_1, S_2) \text{ in $T$}} |S|w(S_1, \overline{S_1}) + |S|w(S_2, \overline{S_2}) - |S|w(S, \overline{S})
    \end{align*}
    In the above sum, if we assign cuts to their respective nodes, then we obtain the following: The root node is assigned $-|S|w(S, \overline{S}) = 0$. Each internal node $S_1$ which is not a leaf node or the root is assigned $|S|w(S_1, \overline{S_1}) - |S_1|w(S_1, \overline{S_1}) = |S_2| w(S_1, \overline{S_1})$, where $S \rightarrow (S_1, S_2)$ is the parent split of $S_1$. Finally, each leaf node $S_1$ is assigned $|S|w(S_1, \overline{S_1}) = |S_2|w(S_1, \overline{S_1}) + w(S_1, \overline{S_1})$, using the fact that $|S_1| = 1$. This brings us to the following decomposition~\citep{agarwal2022sublinear}:
    \[
        \cost_G(T) = \underbrace{\sum_{S \rightarrow (S_1, S_2) \text{ in $T$}}|S_2| w(S_1, \overline{S_1}) + |S_1| w(S_2, \overline{S_2})}_{\cost_G^1(T)} + \underbrace{\sum_{i = 1}^n w(v, \overline{v})}_{\cost_G^2}.
    \]
    We refer to the leftmost term of the above as $\cost_G^1(T)$, and the rightmost term as $\cost_G^2$. Observe the second quantity does not depend on $T$. Now, for any tree $T$, we have
    \begin{align*}
        \cost_{G'}^1(T) &\leq \sum_{S \rightarrow (S_1, S_2) \text{ in $T$}}\Big( |S_2| ((1+\alpha_n)w_{G}(S_1, \overline{S_1}) + \beta_n \min\{|S_1|, n - |S_1|\}) \\ &\qquad + |S_1| ((1+\alpha_n )w_{G}(S_2, \overline{S_2}) + \beta_n \min\{|S_2|, n - |S_2|\})\Big) \\
        &\leq (1+\alpha_n)\cost_{G}^1(T) + \beta_n \sum_{S \rightarrow (S_1, S_2) \text{ in $T$}} |S_2| \min\{|S_1|, n - |S_1|\} + |S_1| \min\{|S_2|, n - |S_2|\} \\
        &\leq (1+\alpha_n)\cost_{G}^1(T) + \beta_n \sum_{S \rightarrow (S_1, S_2) \text{ in $T$}} 2|S_1| |S_2| \\
        &\leq (1+\alpha_n)\cost_{G}^1(T) + \beta_n n^2,
    \end{align*}
    where the final line comes from an induction argument: if $f(n) \leq \max_{1 \leq i \leq n} f(i)f(n-i) + 2\beta i(n-i)$, then we can show via induction that $f(n) \leq \frac{n^2\beta}{2}$. By a similar process, we can show the following inequalities
    \begin{align}
    (1-\alpha_n) \cost_{G}^1(T) - \beta_n n^2 &\leq \cost_{G'}^1(T) \leq (1+\alpha_n) w_G^1(T) + \beta_n n^2 \label{eq:cost-decomp1}\\
    (1-\alpha_n) \cost_{G}^2 - \beta_n n &\leq \cost_{G'}^2 \leq (1+\alpha_n) \cost_{G}^2 + \beta_n n \label{eq:cost-decomp2}
    \end{align}
    This implies that
    \[
        (1-\alpha_n) \cost_{G}(T) - 2\beta_n n^2 \leq \cost_{G'}(T) \leq (1+\alpha_n) \cost_{G}(T) + 2\beta_n n^2.
    \]
    This allows us to derive that
    \begin{align*}
        \cost_G(T') &\leq (1+\alpha_n)\cost_{G'}(T') + 2\beta_n n^2 \\
        &\leq (1+\alpha_n) a_n \cost_{G'}(T^*) + 2\beta_n n^2 \\
        &\leq (1+\alpha_n) a_n ((1+\alpha_n) \cost_{G}^* + 2\beta_n n^2) + 2\beta_n n^2 \\
        &\leq (1+2\alpha_n) a_n \cost_{G}^* + (4a_n + 2) \beta_n n^2
    \end{align*}
    Plugging $T^*$, the optimal tree for $G$, into the above, we obtain that $\cost_{G'}^* \leq (1+\alpha_n)\cost_{G}^* + 2\beta_n n^2$, and therefore,
    \[
        \cost_{G'}(T') \leq a_n (1+\alpha_n) \cost_G^* + 2 a_n \beta_n n^2.
    \]
    We also have that $(1-\alpha_n) \cost_{G}(T') - 2\beta_n n^2 \leq \cost_{G'}(T')$, and we obtain our result by rearranging. 

\subsection{Proof of Lemma~\ref{lem:exp-util}}

Using a general lemma about the exponential mechanism~\citep{mcsherry2007mechanism}, we are able to prove a bound on the algorithm error.
\begin{lem}\label{lem:exp-util-appendix}
Let $f(X,Y)$ be a function with sensitivity $1$ in $X$.
Suppose we run the exponential mechanism $M : \calX \rightarrow \calY$ with finite range $\calY$ using utility function $u_X(Y) = f(X,Y)$. Let $OPT(X) = \min_{Y \in \calY} u_X(Y)$. If our privacy budget is $\epsilon$, then for each $X \in \calX$, we have
\[
    \Pr[u_X(M(X)) \leq OPT(X) + 2\frac{\log(|\calY|)}{\epsilon}] \geq 1-\frac{1}{|\calY|}.
\]
\end{lem}
\begin{proof}
    Let $\calZ = \{Y \in \calY : u_X(Y) \leq OPT(X) + 2\frac{\log(|\calY|)}{\epsilon}\}$. We are guaranteed that the optimal element, $Z^*$, with $u_X(Z^*) = OPT(X)$, is in $\calZ$. We want to lower bound the quantity $\Pr[M(X) \in \calZ]$. Observe that
    \begin{align*}
        \Pr[M(X) \in \calZ] &= \frac{\sum_{Z \in \calZ} e^{-\epsilon u_X(Z)/2}}{\sum_{Z \in \calZ} e^{-\epsilon u_X(Z)/2} + \sum_{Y \in \calY, Y \notin \calZ} e^{-\epsilon u_X(Y)/2}} \\
        &\geq \frac{e^{-\epsilon u_X(Z^*)/2}}{e^{-\epsilon u_X(Z^*)/2} + \sum_{Y \in \calY, Y \notin \calZ} e^{-\epsilon u_X(Y)/2}} \\
        &= \frac{e^{-\epsilon OPT(X)/2}}{e^{-\epsilon OPT(X)/2} + \sum_{Y \in \calY, Y \notin \calZ} e^{-\epsilon u_X(Y)/2}}.
    \end{align*}
    The second line holds because the function $g(z) = \frac{z}{z+K}$ for $K > 0$ is decreasing as $z \rightarrow 0$. The bottom sum can be upper bounded with $|\calY| e^{-\epsilon (OPT(X) + 2\log(|\calY|) / \epsilon)/2} \leq \frac{1}{|\calY|} e^{-\epsilon OPT(X)/2}$. Thus, we are left with
    \[
        \Pr[M(X) \in \calZ] \geq \frac{1}{1 + 1 / |\calY|} \geq 1-\frac{1}{|\calY|}.
    \]
\end{proof}
For hierarchical clustering, our algorithm is a corollary of the previous result:

\begin{proof}
We apply the exponential mechanism with utility function $u_G(T) = -\frac{1}{n} \cost_G(T)$, which has sensitivity $1$. The range of the algorithm is the space of trees with $n$ nodes; there are at most $n^n$ trees of this size. By Lemma~\ref{lem:exp-util-appendix}, the utility satisfies $\Pr[ \frac{\cost_G^*}{n} \leq \frac{\cost_G(M(G))}{n} + 2 \frac{n \log n}{\epsilon}] \geq 1-o(1)$, and hence the algorithm is a $(1, O(\frac{n^2 \log n}{\epsilon}))$-approximation. 

For the revenue objective, we apply the exponential mechanism with utility function $u_G(T) = \frac{1}{2n} \cost_G^{\mw}(T)$, which has sensitivity $1$. By Lemma~\ref{lem:exp-util}, the utility satisfies $\Pr[\frac{\cost_G^{\mw}(M(G))}{2n} \leq \frac{\cost_G^{\mw*}}{2n} + 2 \frac{n \log n}{\epsilon}] \geq 1-o(1)$. This establishes $(1, O(\frac{n^2\log n}{\epsilon}))$-approximation.
\end{proof}

\section{Omitted proofs from Section~\ref{sec:algorithms-hsbm}}

\subsection{Proof of Theorem~\ref{thm:hc-hsbm-util}}\label{sec:hc-hsbm-util}

In order to prove this theorem, we will show that \dphchsbm{} finds a $(1+o(1))$-approximate ground-truth tree, and then appeal to a result showing the such trees are approximately optimal with high probability~\citep{cohen2019hierarchical}:

\begin{lem}\label{lem:approx-hsbm-tree-opt} (Lemma 5.10 from~\citet{cohen2019hierarchical})
Let $G$ be a graph drawn from $\hsbm(B,P,f)$,  where $p_{min} = \min_{i \in B \cup N} f(i) \geq \omega(\sqrt{\frac{\log n}{n}})$. Let $(B,P',f')$ be a $\gamma$-approximate ground-truth tree. Then, with probability $1-2^{-n}$, we have
\[
    \dcost_G(P') \leq \gamma (1 + o(1)) \dcost_G^*,
\]
\end{lem}

We now show that \dphchsbm{} outputs an approximate ground-truth tree.
We introduce a high-probability event and prove that if it happens, then the output is an approximate ground-truth tree.

Our event $\calE$ states that $sim(B_i, B_j)$ as used in \dphchsbm{} is a good estimate for $f(LCA_P(B_i, B_j))$. Intuitively, this makes sense, as if one had access to $f(LCA_P(B_i, B_j))$, then it would be easy to construct $P$ (or an equivalent tree) using single linkage. Formally, we let $\calE$ denote the event that there exists $\alpha$ such that for all $B_i, B_j$,
\begin{equation}\label{eq:good-event}
    \big|sim(B_i, B_j) - f(LCA_P(B_i, B_j))\big| \leq \alpha f(LCA_P(B_i, B_j)).
\end{equation}

The following lemma shows that $\calE$ occurs with high probability.
\begin{lem}\label{lem:good-event}
If $|B_i| \geq n^{2/3}$ for all $i, j$, $\epsilon \geq \frac{1}{n^{1/2}}$, and $f(x) \geq \frac{\log n}{n^{1/2}}$, then the event $\calE$ occurs with $\alpha = \frac{8}{n^{1/6}}$ with probability at least $1-\frac{2}{n}$.
\end{lem}

\begin{proof}
The values $w_G(B_i, B_j)$ are distributed according to $\text{Binomial}(N_{ij}, p_{ij})$, where $N_{ij} = |B_i||B_j|$ and $p_{ij} = f(LCA_P(B_i, B_j))$. By Hoeffding's bound, we have that
\[
    \Pr[|w_G(B_i, B_j) - p_{ij}N_{ij}| \geq 2 \log n \sqrt{N_{ij}}] \leq \frac{1}{n^3}.
\]
Furthermore, we have that $\Pr[|\calL_{ij}| \geq \frac{6 \log n}{\epsilon}] \leq \frac{1}{n^3}$. Plugging in $sim(B_i, B_j) = \frac{w_G(B_i, B_j) + \calL_{ij}}{N_{ij}}$, we obtain
\[
    \Pr\left[|sim(B_i, B_j) - p_{ij}| \geq \frac{2 \log n}{\sqrt{N_{ij}}} + \frac{6 \log n}{\epsilon N_{ij}} \right] \leq \frac{2}{n^3}.
\]
Because $N_{ij} \geq n^{4/3}$ and $\epsilon \geq \frac{1}{n^{1/2}}$, we have $\frac{2 \log n}{\sqrt{N_{ij}}} + \frac{6 \log n}{\epsilon N_{ij}} \leq \frac{8 \log n}{n^{2/3}} \leq \frac{8}{n^{1/6}} p_{ij}$.
Thus, we obtain $\Pr[|sim(B_i, B_j) - p_{ij}| \geq \alpha p_{ij}] \leq \frac{2}{n^3}$, with $\alpha = \frac{8}{n^{1/6}}$. Taking a union bound over all $\binom{k}{2} \leq n^2$ choices of $i,j$, we obtain our result.
\end{proof}

Finally, we show that when $\calE$ occurs, then \dphchsbm{} finds an approximate ground-truth tree. A similar result was proved in~\citet{cohen2019hierarchical}, though our lemma statement is sufficiently different that we include a proof here.
\begin{lem}\label{lem:approx-ground-truth}
Assume that event $\calE$ occurs. Then, the tuple $(B, T, f')$ returned by Algorithm~\ref{alg:priv-hc-hsbm} is a $(1 + \alpha)$-approximate ground-truth tree for $(B, P, f)$.
\end{lem}
\begin{proof}
We want to show that for all $B_i, B_j \in V$, we have
\[
    (1-\alpha) f(LCA_P(B_i, B_j)) \leq f'(LCA_{P'}(B_i,B_j)) \leq (1+\alpha) f(LCA_P(B_i, B_j)).
\]
Let $I = LCA_{T}(B_i, B_j)$ be the internal node in which $B_i, B_j$ are merged, and let $C_i, C_j$ be the children of $I$ such that $B_i \subseteq C_i$ and $B_j \subseteq C_j$.
We have that
\[
    f'(LCA_{P'}(B_i, B_j)) = sim(C_i, C_j) = \max_{B \in C_i, B' \in C_j} sim (B, B').
\]
Thus, it holds that $sim(B_i, B_j) \leq f'(LCA_{P'}(B_i, B_j))$. As event $\calE$ holds, we have that $sim(B_i, B_j) \geq (1-\alpha)f(LCA_P(B_i, B_j))$.

To finish, we show that $sim(C_i, C_j) \leq (1+\alpha) f(LCA_P(B_i, B_j)$.
Let $J = LCA_P(B_i, B_j)$ be the internal node in which $B_i, B_j$ are merged in $P$, and let $D_i, D_j$ be the children of $J$ such that $B_i \subseteq D_i$ and $B_j \subseteq D_j$. We consider the following two cases.
\paragraph{Case 1:} $C_i \subseteq D_i$ and $C_j \subseteq D_j$.
Then, we have 
\[
    sim(C_i, C_j) \leq \max_{B \in D_i, B' \in D_j} sim(B, B') \leq (1+\alpha) \max_{B \in D_i, B' \in D_j} f(LCA_P(B, B')).
\]
As $D_i, D_j$ are nodes of the ground-truth tree, it holds that $f(LCA_P(B, B'))$ is the same for any choice of $B \in D_i, B' \in D_j$. In particular, this is true for $f(LCA_P(B_i, B_j))$.

\paragraph{Case 2:} There exists $B_\ell$ such that $B_\ell \subseteq C_i$ and $B_\ell \nsubseteq D_i$ (or the same holds for $C_i, D_i$ replaced by $C_j, D_j$).
WLOG, suppose the former case holds. Then, there exists a child $N$ of $I$ whose children are $N_L, N_R$, such that $N_L \subseteq D_i$ and $N_R \cap D_i = \emptyset$. It then follows that 
\[
    sim(N_L, N_R) \leq (1+\alpha) \max_{B \in N_L, B' \in N_R} f(LCA_P(B, B')) \leq (1+\alpha) f(LCA_P(B_i, B_j)),
\]
where the second inequality holds because $f$ is decreasing as we ascend $P$.
However, we also have that $sim(N_L, N_R) \geq sim(C_i, C_j)$, as $sim$ also obeys this property (if the last inequality did not hold, then $N_L, N_R$ would not have been merged). This finishes the last case.
\end{proof}

The proof follows by applying Lemma~\ref{lem:good-event} and then Lemma~\ref{lem:approx-ground-truth}.

\subsection{Proof of Theorem~\ref{thm:com-hsbm-priv}}\label{sec:com-hsbm-priv}
\subsubsection{Overview}
When running \dpcom{}, fix $Y,Z_1,Z_2$, and let $(\hat{A}_1, \hat{A}_2)$ and $(\hat{A}_1', \hat{A}_2')$ be the splits of $\hat{A}$ and an adjacent database $\hat{A}'$.
We will view the matrix $F = P(\Pi_{\hat{A}_1}^{(k)}(\hat{A}_2))$ as a vector, and then show that releasing $F$ plus appropriate Gaussian noise satisfies privacy via the Gaussian mechanism. Our proof will bound the $L_2$ sensitivity of $F$, given by
\[
    \Delta_2(F) = \|P(\Pi_{\hat{A}_1}^{(k)}(\hat{A}_2)) - P(\Pi_{\hat{A}_1'}^{(k)}(\hat{A}_2'))\|_F,
\]
in terms of the quantity $\Gamma = \frac{\sigma_1(\hat{A}_2)}{\sigma_k(\hat{A}_1) - \sigma_k(\hat{A}_2)}$. Recall that $P$ is a random $m \times \frac{n}{2}$ projection matrix. To control this sensitivity, we will need the fact that $P$ preserves the distances in $A$ via the Johnson-Lindenstrauss projection theorem:
\begin{thm}\label{thm:jl}
(Johnson-Lindenstrauss projection theorem~\citep{johnson1984extensions}): Let $0 \leq \alpha < \frac{1}{2}$ and $0 \leq \beta \leq 1$, and $m = 8\frac{ \ln \frac{2}{\beta}}{\alpha^2}$. If $x \in \R^n$ is a vector and $P \sim \calN(0, \frac{1}{\sqrt{m}})^{m \times n}$ is a random matrix then with probability $1-\beta$, we have
\[
(1-\alpha) \|x\|_2 \leq \|Px\|_2 \leq (1+\alpha) \|x\|_2
\]
\end{thm}
We use the above theorem to show that the matrix $P$ does not increase the sensitivity $\Delta(F)$ with high probability.
\begin{lem}\label{lem:proj-sens}
    Let $0 \leq \delta < 1$ and $m = 64 \ln \frac{2n}{\delta}$. Then, if $P \sim \calN(0, \frac{1}{\sqrt{m}})^{m \times n/2}$ the following holds with probability at least $1-\frac \delta 4$:
    \[
        \Delta(F) \leq \frac 3 2 \|\Pi_{\hat{A}_1}^{(k)}(\hat{A}_2) - \Pi_{\hat{A}_1'}^{(k)}(\hat{A}_2')\|_F.
    \]
\end{lem}
\begin{proof}
Let the columns of $\Pi_{\hat{A}_1}^{(k)}(\hat{A}_2)$ be $\{a_1, \ldots, a_{n/4}\}$ and the columns of $\Pi_{\hat{A}_1'}^{(k)}(\hat{A}_2')$ be $\{a_1', \ldots, a_{n'/4}\}$. By the union bound, Theorem~\ref{thm:jl} with $\alpha=\frac{1}{2}$ and $\beta = \frac{\delta}{n}$ applies to all vectors $a_i - a_i'$ with probability at least $1-\frac{\delta}{4}$. Thus, we have
\[
\Delta_2(F)^2 = \sum_{i=1}^{n/4} \|P(a_i) - P(a_i')\|_2^2 \leq (1+\alpha)^2 \sum_{i=1}^{n/4} \|a_i - a_i'\|_2^2 = (1+\alpha)^2  \|\Pi_{\hat{A}_1}^{(k)}(\hat{A}_2) - \Pi_{\hat{A}_1'}^{(k)}(\hat{A}_2')\|_F^2.
\]
The result follows.
\end{proof}
Finally, we need a bound on the stability of the projection $\Pi_{\hat{A}_1}^{(k)}$ when $\hat{A}_1$ is perturbed. This is the result of the Davis-Kahan Theorem~\citep{bhatia1997}.
\begin{thm}\label{thm:davis-kahan}
Let $\hat{A}_1, \hat{A}_1'$ be matrices where $d_k= \sigma_k(\hat{A}_1) - \sigma_{k+1}(\hat{A}_1) > 0$. Then, 
\[
    \|\Pi_{\hat{A}_1}^{(k)} - \Pi_{\hat{A}_1'}^{(k)}\|_F \leq \frac{\|\hat{A}_1 - \hat{A}_1'\|_F}{d_k}.
\]
Furthermore, the above holds replacing $\|\cdot\|_F$ with $\|\cdot \|_2$.
\end{thm}

Having bounded the $L_2$-sensitivity, we finally use the well-known Gaussian mechanism~\citep{dwork2014algorithmic}
\begin{thm}\label{thm:gauss-mech}
    If $x \in \R^m$ has $L_2$ sensitivity at most $S$, then releasing $x + N$, where $N \sim \frac{S}{\epsilon} \sqrt{2 \ln \frac{1.25}{\delta}}\calN(0,1)^{m}$ satisfies $(\epsilon, \delta)$-DP.
\end{thm}

\subsubsection{Proof}
Let $\hat{A}$ and $\hat{A}'$ be two adjacent inputs, and consider two runs of \dpcom{} with fixed $Y,Z_1,Z_2$, and $P$; we will show that the outputs satisfy $(\epsilon, \delta)$-DP. Let $\hat{A}_1'$ and $\hat{A}_2'$ be the values of $\hat{A}_1$ and $\hat{A}_2$ when $\hat{A}'$ is used instead of $\hat{A}$. \dpcom{} can be viewed as a post-processing of the private release of values $d_k = \sigma_k(\hat{A}_1) - \sigma_{k+1}(\hat{A}_1)$, $\sigma_1(\hat{A}_2)$, and $F$; thus, we will show that releasing each of these values satisfies privacy.

Using Lindskii's inequality~\citep{bhatia1997}, each rank $i$ singular value of $\hat{A}_1, \hat{A}_2$ can only change by $1$ when $\hat{A}$ is changed to $\hat{A}'$. Thus, the sensitivity of $d_k$ is $2$, of $\sigma_1$ is $1$, and thus the release of $\tilde{d}_k = d_k + \frac{8}{\epsilon} \ln \frac{4}{\delta} + Lap(\frac{8}{\epsilon})$ and $\tilde{\sigma}_1 = \sigma_1 + \frac{4}{\epsilon} \ln \frac{4}{\delta} + Lap(\frac{4}{\epsilon})$ both satisfy $(\frac{\epsilon}{4}, 0)$-DP. Thus, we will show that releasing $\tilde{F}$ satisfies $(\frac \epsilon 2, \delta)$-DP, and privacy will follow by composition.

By Lemma~\ref{lem:proj-sens} with probability at least $1-\frac{\delta}{4}$, we have
\[
    \Delta_2(F) \leq \tfrac 3 2 \|\Pi_{\hat{A}_1}^{(k)}(\hat{A}_2) - \Pi_{\hat{A}_1'}^{(k)}(\hat{A}_2')\|_F
\]
We have either $\hat{A}_1 = \hat{A}_1'$ or $\hat{A}_2 = \hat{A}_2'$. We analyze the cases separately.

\paragraph{Case $\hat{A}_1 = \hat{A}_1'$:} Then, $\hat{A}_2$ and $\hat{A}_2'$ differ in one bit, so $\hat{A}_2 = \hat{A}_2' + E$, where $E$ is a matrix that is $\pm 1$ in one entry and $0$ everywhere else.
Then,
\[
\tfrac 3 2\|\Pi_{\hat{A}_1}^{(k)}(\hat{A}_2) - \Pi_{\hat{A}_1}^{(k)}(\hat{A}_2')\|_F = \tfrac 3 2\|\Pi_{\hat{A}_1}^{(k)}(E)\|_F \leq \tfrac 3 2\|E\|_F \leq \tfrac 3 2,
\]
where the inequality holds because projecting vectors onto a subspace cannot increase their magnitude.

\paragraph{Case $\hat{A}_2 = \hat{A}_2'$:} Then, $\hat{A}_1$ and $\hat{A}_1'$ differ in one bit, so $\|\hat{A}_1 - \hat{A}_1'\|_F \leq 1$. We have
\[
\|\Pi_{\hat{A}_1}^{(k)}(\hat{A}_2) - \Pi_{\hat{A}_1'}^{(k)}(\hat{A}_2')\|_F \leq 2k \|(\Pi_{\hat{A}_1}^{(k)} - \Pi_{\hat{A}_1'}^{(k)})(\hat{A}_2)\|_2 \leq 2k\|\Pi_{\hat{A}_1}^{(k)} - \Pi_{\hat{A}_1'}^{(k)}\|_2\|\hat{A}_2\|_2,
\]
where the first inequality holds because each term has rank at most $k$, so the entire quantity has rank at most $2k$, and the  second holds by sub-multiplicativity of $\|\cdot\|_2$.
By Theorem~\ref{thm:davis-kahan}, we have $\|\Pi_{\hat{A}_1}^{(k)} - \Pi_{\hat{A}_1'}^{(k)}\|_2 \leq \frac{1}{d_k}$.
Thus, we have
\[
\frac 3 2 \|\Pi_{\hat{A}_1}^{(k)}(\hat{A}_2) - \Pi_{\hat{A}_1'}^{(k)}(\hat{A}_2')\|_F \leq \frac{3k\|\hat{A}_2\|_2}{d_k} = 3k\Gamma.
\]
By concentration of Laplace variables, we have $\tilde{d}_k \leq d_k$ and $\tilde{\sigma}_1 \geq \sigma_1$, so $\Gamma \leq \frac{\tilde{\sigma}_1}{\tilde{d}_k} = \tilde{\Gamma}$ with probability at least $1-\frac{\delta}{2}$. Thus, the sensitivity $\Delta(F)$ is at most $3k\tilde{\Gamma}$, and $(\tfrac \epsilon 2, \frac{\delta}{4})$-DP follows via Theorem~\ref{thm:gauss-mech}. Factoring in the aformentioned failure probabilities, the entire release of $\tilde{F}$ satisfies $(\tfrac \epsilon 2, \delta)$-DP.

\subsection{Proof of Corollary~\ref{thm:com-hsbm-util-inf}}\label{sec:com-hsbm-util}
\subsubsection{Overview}

Recall that \dpcom{} sees a matrix $\hat{A}$ drawn from $\hsbm(B,P,f)$, with expectation matrix $A$. We define $\tau^2 = \max f(x)$, $s = \min_{i=1}^k |B_i|$, and $\Delta=\min_{u \in B_i, v \in B_j, i \neq j} \|A_u - A_v\|_2$. 
We will show that \dpcom{} approximates $\Pi_{\hat{A}_1}^{(k)}(\hat{A}_2)$, which is guaranteed to cluster the original communities via the following result~\citep{vu2014simple}. We let the columns of $\Pi_{\hat{A}_1}^{(k)}(\hat{A}_2)$, which is indexed by the set $Z_2$, be $\{b_i : i \in Z_2\}$.

\begin{thm}\label{thm:svd-recovery} (\citet{vu2014simple}): There exists a universal constant $C$ such that if $\tau^2 \geq C \frac{\log n}{n}$, $s \geq C \log n$, and $k < n^{1/4}$.  $\Delta > C (\tau \sqrt{\frac{n}{s}} + \tau \sqrt{k \log n} + \frac{\tau\sqrt{nk}}{\sigma_k(A)})$, with probability at least $1 - n^{-1}$, then the columns $\{b_i : i \in Z_2\}$ in $\Pi_{\hat{A}_1}^{(k)}(\hat{A}_2)$ satisfy:
\begin{align*}
    \|b_{i} - b_{j}\|_2 &\leq \frac{\Delta}{4}  \ \ \ \text{if $\exists u.~i \in B_u, j \in B_u$ (i.e. $i,j$ are in the same community)} \\ 
    \|b_{i} - b_{j}\|_2 &\geq \Delta  \ \ \ \text{otherwise.}
\end{align*}
\end{thm}

Thus, the clusters in $\Pi_{\hat{A}_1}^{(k)}(\hat{A}_2)$ cluster the original communities assuming $\Delta$ is large enough. We will show that $\tilde{F}$ clusters the original communities assuming some condition on $\Delta$. Since
\dpcom{} returns $\tilde{F} = P(\Pi_{\hat{A}_1}^{(k)}(\hat{A}_2)) + N$, where $N$ is Gaussian noise, our proof involves showing that the distances in $\tilde{F}$ approximate those in $\Pi_{\hat{A}_1}^{(k)}$ using the Johnson-Lindenstrauss lemma and concentration of the Gaussian noise.

We formally restate Theorem~\ref{thm:com-hsbm-util-inf}:

\begin{thm}\label{thm:com-hsbm-util}
Let $\hat{A}$ be drawn from $\hsbm(B,P,f)$. There is a universal constant $C > 2000$ such that if $\tau^2 \geq C \frac{\log n}{n}$, $s \geq C \sqrt{n \log n}$, $k < n^{1/4}$, $\delta < \frac{1}{n}$, $\sigma_k(A) \geq C \max\{ \tau \sqrt{n}, \frac{1}{\epsilon} \ln \frac{4}{\delta}\}$, and 
\[\Delta > C\max\left\{\tfrac{ k (\ln \frac{1}{\delta})^{3/2}}\epsilon \tfrac{\sigma_1(A)}{\sigma_{k}(A)},  \tau \sqrt{\tfrac{n}{s}} + \tau \sqrt{k \log n} + \tfrac{\tau\sqrt{nk}}{\sigma_k}\right\},\]
then with probability at least $1 - 3n^{-1}$, \dpcom{} returns a set of points $\tilde{F} = \{f_i : i \in Z_2\}$ such that
\begin{align*}
    \|f_{i} - f_{j}\|_2 &\leq \frac{2\Delta}{5}  \ \ \ \text{if $\exists u.~i,j \in B_u$} \\
    \|f_{i} - f_{j}\|_2 &\geq \frac{4\Delta}{5}  \ \ \ \text{otherwise},
\end{align*}
and thus the clusters in $\tilde{F}$ indicate the communities.
\end{thm}

\subsubsection{Proof of Theorem~\ref{thm:com-hsbm-util}}
Let the columns of $\tilde{F}$ be $\{f_i : i \in Z_2 \}$.
We have $f_i = P(b_i) + n_i$, where $n_i \sim \frac{3k \tilde{\Gamma}}{\epsilon} \sqrt{2\ln \frac{5}{\delta}}N(0,1)^{m}$. By concentration bounds, we have with probability $1-\frac{1}{n}$ that each $n_i$ satisfies $\|n_i\|_2 \leq \frac{3k\tilde{\Gamma}}{\epsilon}\sqrt{2\ln \frac{5}{\delta}} \sqrt{2m \ln n} \triangleq K$. Next, applying Theorem~\ref{thm:jl} on the vectors $b_i - b_j$ for $i,j \in Z_2$, we have $0.9 \|b_i - b_j\|_2 \leq \|P(b_i) - P(b_j)\|_2 \leq 1.1 \|b_i - b_j\|_2$ with probability $1-\delta > 1-\frac{1}{n}$. Thus, if $\exists u.~ i \in B_u, j \in B_u$, then
\begin{align*}
    \|f_i - f_j\|_2 &\leq \|P(b_i) - P(b_j)\|_2 + \|n_i\|_2 + \|n_j\|_2 \\
    &\leq 1.1 \|b_i - b_j\|_2 + 2K \\
    &\leq 0.275 \Delta + 2K,
\end{align*}
Otherwise, we have
\begin{align*}
    \|f_i - f_j\|_2 &\geq \|P(b_i) - P(b_j)\|_2 - \|n_i\|_2 - \|n_j\|_2 \\
    &\geq 0.9 \|b_i - b_j\|_2 - 2K \\
    &\geq0.9\Delta - 2K.
\end{align*}

Finally, we show that $K$ can be upper bounded by the singular values of the expectation matrix $A$. This can be done with the following two lemmas which are proven implicitly in~\citet{vu2014simple}.
\begin{lem}\label{lem:rand-mat-sv}
Let $A$ be an $m \times n$ (with $m \geq n$) matrix of expectations in $[0,1]$, and let $\hat{A}$ be a randomized rounding of $A$ to $\{0,1\}$. Then, with probability at least $1 - \frac{1}{n}$, we have for all $1 \leq i \leq m$, $|\sigma_i(A) - \sigma_i(\hat{A})| \leq 4\tau \sqrt{n} + 4 \log n$, where $\tau^2$ is the maximum probability in $A$.
\end{lem}

\begin{proof}
Each $\sigma_{i+1}(A)$ is equal to $\max_{\text{rank}(A_{i}) = i} \|A - A_i\|_2$.
Let $A_i^*, \hat{A}_i^*$ be rank $i$ matrices such that $\sigma_{i+1}(A) = \|A - A_i^*\|_2$ and $\sigma_{i+1}(\hat{A}) = \|\hat{A} - \hat{A}_i^*\|_2$. We have that $\sigma_{i+1}(A) \leq \|A - \hat{A}_i^*\|_2 \leq \|\hat{A} - \hat{A}_i^*\|_2 + \|A - \hat{A}\|_2$. 

Thus, it remains to bound $\|A - \hat{A}\|_2$. Let the columns in $A - \hat{A}$ be $a_1, \ldots, a_n$. Using Lemma 7 from~\citet{vu2005spectral}, we have that with probability at least $1-\frac{1}{n^3}$, the length of the projection of $a_i$ onto a basis vector $e_i$ is at most $4(\tau + \frac{\log n}{\sqrt{n}})$. Thus, the total length $\|Ae_i\|_2$ is at most $4(\tau + \frac{\log n}{\sqrt{n}}$, and thus $\|A\|_2 \leq \sqrt{n}4(\tau + \frac{\log n}{\sqrt{n}})$ establishing that $\sigma_{i+1}(A) \leq \sigma_{i+1}(\hat{A}) + 4\sqrt{n}\tau + 4 \log n$. Likewise, we can show that $\sigma_{i+1}(A) \geq \sigma_{i+1}(\hat{A}) - 4\sqrt{n}\tau - 4 \log n$.
\end{proof}

\begin{lem}\label{lem:subsample-sv}
Let $A$ be an expectation matrix of $\hsbm(B,P,f)$ with $k$ blocks with minimum block size $s \geq 16\sqrt{n \log n}$, and let $C$ be the submatrix of $A$ with rows $Y$ and columns $Z$, where $|Y| = \frac{n}{2}$ and $|Z| = \frac{n}{4}$ are chosen randomly from $[n]$ such that $Y \cap Z = \emptyset$. Then, with probability at least $1-\frac{1}{n}$, for all $1 \leq i \leq k$, we have
\[
(\tfrac 1 8 - \tfrac{\sqrt{n \log n}} s)\sigma_i(A_1) \leq \sigma_i(A_1) \leq (\tfrac 1 8 + \tfrac{\sqrt{n \log n}} s)\sigma_i(A_1)
\]
\end{lem}
\begin{proof}
Observe that the blocks in $C$ are indexed in rows by $B_1 \cap Y, \ldots, B_k \cap Y$ and in columns by $B_1 \cap Z, \ldots, B_k \cap Z$. By Chernoff's bound, with probability at least $1 - \frac{1}{n^2}$, we have for all $i$ that
\[
    \frac{1}{2} - \frac{\sqrt{n \log n}}{|B_i|} \leq \frac{|B_i \cap Y|}{|B_i|} \leq \frac{1}{2} + \frac{\sqrt{n \log n}}{|B_i|} \ \ \ \ \ \ \ \ \ \frac 1 4 - \frac{\sqrt{n \log n}}{|B_i|} \leq \frac{|B_i \cap Z|}{|B_i|} \leq \frac 1 4 + \frac{\sqrt{n \log n}}{|B_i|}.
\]

We have $\sigma_{k}(A) = \min_{\text{rank}(A_{k-1}) = k-1} \|A - A_{k-1}\|_F$ and $\sigma_{k}(C) = \min_{\text{rank}(C_{k-1}) = k-1} \|C-C_{k-1}\|_F$; let $A_{k-1}^*$ and $C_{k-1}^*$ be the maximizers of the previous expressions. Let $A'$ denote the matrix $C_{k-1}^*$ with rows and columns duplicated such that each element $(A')_{ij}$ is equal to $(C_{k-1}^*)_{xy}$, where $x,y$ are any two points in the same block as $i,j$, respectively. Accounting for the duplication factors of each block, we have
\[
    \left(\frac 1 2 - \frac{\sqrt{n \log n}}{s}\right)\left(\frac 1 4 - \frac{\sqrt{n \log n}}{s}\right)\|A - A'\|_F \leq  \|C - C_{k-1}^*\|_F,
\]
and thus we see that $(\frac 1 8 - \frac{\sqrt{n \log n}}{s})\sigma_k(A) \leq \sigma_k(A_1)$. By a similar sampling argument, we can show that $(\frac 1 8 + \frac{\sqrt{n \log n}}{s})\sigma_k(A) \geq \sigma_k(A_1)$. Repeating the argument for $\sqrt{\sigma_{i}(A)^2 + \cdots \sigma_{k}(A)^2} = \min_{\text{rank}(A_{i-1})=i-1}\|A - A_{i-1}\|_F$, we obtain the result for all $1 \leq i \leq k$.
\end{proof}
Let $A_1, A_2$ be the expectation matrices of $\hat{A}_1, \hat{A}_2$ for fixed $Y,Z_2$. Using Lemmas~\ref{lem:rand-mat-sv} and~\ref{lem:subsample-sv}, we have that $\sigma_{1}(\hat{A}_2) \leq \sigma_1(A_2) + 4\tau \sqrt{n} + 4 \log n \leq (\frac{1}{8} + \frac{\sqrt{n \log n}}{s})\sigma_1(A) + 4\tau \sqrt{n} + 4 \log n \leq \frac{3}{32} \sigma_1(A) + 4\tau \sqrt{n} + 4 \log n$.
Applying these again, we obtain
\begin{align*}
    d_k(\hat{A}_1) &= \sigma_{k}(\hat{A}_1) - \sigma_{k+1}(\hat{A}_1) \\
    &\geq \sigma_k(A_1) - \sigma_{k+1}(A_1) - 8\tau \sqrt{n} - 8\log n \\
    &\geq (\frac{1}{8}-\frac{\sqrt{n \log n}}{s})(\sigma_k(A) - \sigma_{k+1}(A)) - 8\tau \sqrt{n} - 8\log n \\
    &\geq \frac{1}{16}\sigma_k(A) - 8\tau \sqrt{n} - 8\log n
\end{align*}
Finally, we have $\tilde{\Gamma} = \frac{\tilde{\sigma}_1(\hat{A_2})}{\tilde{d}_k(\hat{A_1})}$, which with probability at least $\delta$, will satisfy
\[
\tilde{\Gamma} \leq \frac{\sigma_1(\hat{A_2}) + \frac{8}{\epsilon} \ln \frac{4}{\delta} }{d_k(\hat{A}_1) - \frac{16}{\epsilon} \ln \frac{4}{\delta}} \leq 
\frac{\frac{3}{32} \sigma_1(A) + 4\tau \sqrt{n} + 4 \log n + \frac{8}{\epsilon} \ln \frac{4}{\delta} }{\frac{1}{16}d_k(A) - 8\tau \sqrt{n} - 8\log n - \frac{16}{\epsilon} \ln \frac{4}{\delta}}.
\]
By our assumption that $\sigma_k(A) \geq 1024 \max \{\tau \sqrt{n}, \frac{1}{\epsilon} \ln \frac{4}{\delta} \}$, we obtain that $\hat{\Gamma} \leq 4 \frac{\sigma_1(A)}{\sigma_k(A)}$, 
This implies that 
\[
    K \leq \frac{12k \sqrt{m\ln \frac{5}{\delta} \ln n}}{\epsilon} \frac{\sigma_1(A)}{\sigma_k(A)} = \frac{48k \sqrt{2 \ln \frac{2n}{\delta} \ln \frac{5}{\delta} \ln n}}{\epsilon} \frac{\sigma_1(A)}{\sigma_k(A)}
    \leq \frac{96k (\ln \frac{5}{\delta})^{3/2}}{\epsilon} \frac{\sigma_1(A)}{\sigma_k(A)},
\]
where the last step follows because $\delta < \frac{1}{n}$. From our assumption, we have $2K \leq 0.1\Delta$, and the result follows.

\subsection{Proof of Corollary~\ref{cor:com-hsbm-util}}

In this special case, we can write $A = P \otimes \textbf{1}_B$, where $P$ is a $k \times k$ matrix with $p$ on the diagonal and $q$ everywhere else, $\textbf{1}_s$ is a $s \times s$ matrix consisting of all $1$s, and $\otimes$ denotes the Kronecker product. It is easy to see that the eigenvalues of $P$ are $\{p + q(k-1), p-q, \ldots, p-q\}$, and the eigenvalues of $\textbf{1}_s$ are $\{s, 0, \ldots, 0\}$. The eigenvalues of $A$ are the product of the two sets of eigenvalues of $P$ and $\textbf{1}_s$. Thus, the top $k$ largest eigenvalues are $s(p + q(k-1))$ and then $k-1$ copies of $s(p-q)$.

Thus, the following properties of $A$ hold: (1) $\sigma_1 = s(p + q(k-1)) ] \leq sk (p + q)$, (2) $\sigma_k = s(p-q)$, (3) $\tau = \sqrt{p}$, and (4) $\Delta = (p-q)\sqrt{s}$. We are able to apply Theorem~\ref{thm:com-hsbm-util} when 
\begin{align*}
    (p-q) \sqrt{s} &\geq \frac{s(p+q)}{s(p-q)} \frac{Ck(\log \frac{1}{\delta})^{3/2}}{\epsilon} \\
    \frac{(p-q)^2}{p+q} &\geq \frac{C(k\log \frac{1}{\delta})^{3/2}}{\sqrt{n}}.
\end{align*}
This establishes the result.

\end{document}